\documentclass[10pt]{article}

\usepackage{amsthm}
\usepackage{mathtools}
\usepackage{amsmath}
\usepackage{bbm}
\usepackage{amsfonts}
\usepackage{amssymb}
\usepackage{xpatch}
\usepackage{enumitem}
\usepackage{graphicx}
\usepackage{thmtools,thm-restate}
\usepackage{multirow}
\usepackage{wrapfig}
\usepackage{tikz}
\usepackage{tkz-graph}
\tikzstyle{vertex}=[circle, draw, inner sep=0pt, minimum size=6pt]
\usetikzlibrary{arrows}

\usepackage[left=1in,right=1in,top=1in,bottom=1in]{geometry}

\usepackage{footnote}
\makesavenoteenv{tabular}

\newtheorem{lemma}{Lemma}

\newtheorem{definition}{Definition}

\newcommand{\R}{\mathbb{R}}
\newcommand{\E}{\mathbb{E}}

\newcommand{\defeq}{\coloneqq}

\DeclarePairedDelimiter{\abs}{\lvert}{\rvert} %
\DeclarePairedDelimiter{\brk}{[}{]}
\DeclarePairedDelimiter{\crl}{\{}{\}}
\DeclarePairedDelimiter{\prn}{(}{)}
\DeclarePairedDelimiter{\nrm}{\|}{\|}
\DeclarePairedDelimiter{\norm}{\|}{\|}
\DeclarePairedDelimiter{\tri}{\langle}{\rangle}

\DeclareMathOperator*{\argmin}{arg\,min}

\newcommand{\mc}[1]{\mathcal{#1}}

\def\ddefloop#1{\ifx\ddefloop#1\else\ddef{#1}\expandafter\ddefloop\fi}
\def\ddef#1{\expandafter\def\csname 
bb#1\endcsname{\ensuremath{\mathbb{#1}}}}
\ddefloop ABCDEFGHIJKLMNOPQRSTUVWXYZ\ddefloop
\def\ddefloop#1{\ifx\ddefloop#1\else\ddef{#1}\expandafter\ddefloop\fi}
\def\ddef#1{\expandafter\def\csname 
b#1\endcsname{\ensuremath{\mathbf{#1}}}}
\ddefloop ABCDEFGHIJKLMNOPQRSTUVWXYZ\ddefloop
\def\ddef#1{\expandafter\def\csname 
c#1\endcsname{\ensuremath{\mathcal{#1}}}}
\ddefloop ABCDEFGHIJKLMNOPQRSTUVWXYZ\ddefloop
\def\ddef#1{\expandafter\def\csname 
h#1\endcsname{\ensuremath{\widehat{#1}}}}
\ddefloop ABCDEFGHIJKLMNOPQRSTUVWXYZ\ddefloop
\def\ddef#1{\expandafter\def\csname 
hc#1\endcsname{\ensuremath{\widehat{\mathcal{#1}}}}}
\ddefloop ABCDEFGHIJKLMNOPQRSTUVWXYZ\ddefloop
\def\ddef#1{\expandafter\def\csname 
t#1\endcsname{\ensuremath{\widetilde{#1}}}}
\ddefloop ABCDEFGHIJKLMNOPQRSTUVWXYZ\ddefloop
\def\ddef#1{\expandafter\def\csname 
tc#1\endcsname{\ensuremath{\widetilde{\mathcal{#1}}}}}
\ddefloop ABCDEFGHIJKLMNOPQRSTUVWXYZ\ddefloop

\newcommand{\indicatorb}[1]{\mathbbm{1}_{\crl*{#1}}}    

\usepackage{nicefrac}

\newsavebox\CBox

\newcommand{\ind}[1]{^{(#1)}}

\newcommand{\pp}[1]{\brk*{#1}_+}

\newcommand{\inner}[2]{\left\langle #1,\, #2 \right\rangle}

\newcommand{\natinote}[1]{{\color{blue}[Nati: {#1}]}}
\newcommand{\removed}[1]{}

\usepackage{booktabs}

\usepackage[numbers, sort, compress]{natbib}

\usepackage[colorlinks,allcolors=blue]{hyperref}
\usepackage{cleveref}

\newcommand{\bx}{\bar{x}}
\newcommand{\bxt}{\bar{x}_t}

\newcommand{\bgt}{\bar{g}_t}
\newcommand{\mkandr}{$M$, $K$, and $R$}
\renewcommand{\P}{\mathbb{P}}

\title{\vspace{-2em}\rule{\linewidth}{1.5pt}\\ \textbf{Is Local SGD Better than Minibatch SGD?} \\\rule[8pt]{\linewidth}{1pt}\vspace{-0.7em}}
\author{\normalsize
\begin{minipage}{0.23\textwidth}
\centering
\textbf{Blake Woodworth}\\ 
\small Toyota Technological Institute at Chicago \\
\small\url{blake@ttic.edu}
\end{minipage}
\begin{minipage}{0.23\textwidth}
\centering
\textbf{Kumar Kshitij Patel}\\
\small Toyota Technological Institute at Chicago \\
\small\url{kkpatel@ttic.edu}
\end{minipage}
\begin{minipage}{0.27\textwidth}
\centering
\textbf{Sebastian U. Stich}\\
\small EPFL\\
\small\url{sebastian.stich@epfl.ch} \\
${}$
\end{minipage}
\begin{minipage}{0.23\textwidth}
\centering
\textbf{Zhen Dai}\\
\small University of Chicago\\
\small\url{zhen9@uchicago.edu}\\
${}$
\end{minipage}\\\\
\normalsize
\begin{minipage}{0.21\textwidth}
\centering
\textbf{Brian Bullins}\\
\small Toyota Technological Institute at Chicago\\
\small\url{bbullins@ttic.edu}
\end{minipage}
\begin{minipage}{0.21\textwidth}
\centering
\textbf{H. Brendan McMahan}\\
\small Google\\
\small\url{mcmahan@google.com}
\end{minipage}
\begin{minipage}{0.27\textwidth}
\centering
\textbf{Ohad Shamir}\\
\small Weizmann Institute \\ of Science \\
\small\url{ohad.shamir@weizmann.ac.il}
\end{minipage}
\begin{minipage}{0.21\textwidth}
\centering
\textbf{Nathan Srebro}\\
\small Toyota Technological Institute at Chicago \\
\small\url{nati@ttic.edu}
\end{minipage}\vspace{-1em}
}
\date{}

\begin{document}
\maketitle
\begin{abstract}
	We study local SGD (also known as parallel SGD and federated averaging), a natural and frequently used stochastic distributed optimization method. Its theoretical foundations are currently lacking and we highlight how all existing error guarantees in the convex setting are dominated by a simple baseline, minibatch SGD. (1) For quadratic objectives we prove that local SGD strictly dominates minibatch SGD and that accelerated local SGD is minimax optimal for quadratics; (2) For general convex objectives we provide the first guarantee that at least \emph{sometimes} improves over minibatch SGD; (3) We show that indeed local SGD does \emph{not} dominate minibatch SGD by presenting a lower bound on the performance of local SGD that is worse than the minibatch SGD guarantee.
\end{abstract}
\section{Introduction}
It is often important to leverage parallelism in order to tackle large scale stochastic optimization problems. A prime example is the task of minimizing the loss of machine learning models with millions or billions of parameters over enormous training sets. 

One popular distributed approach is local stochastic gradient descent (SGD) \citep{zinkevich2010parallelized,Coppola2015:IPM,Zhou2018:Kaveraging,stich2018local}, also known as ``parallel SGD'' or ``Federated Averaging''\footnote{Federated Averaging is a specialization of local SGD to the federated setting, where (a) data is assumed to be heterogenous (not i.i.d.) across workers, (b) only a handful of clients are used in each round, and (c) updates are combined with a weighted average to accommodate unbalanced datasets.} \citep{mcmahan2016communication}, which is commonly applied to large scale convex and non-convex stochastic optimization problems, including in data center and ``Federated Learning'' settings \citep{kairouz2019advances}.  Local SGD uses $M$ parallel workers which, in each of $R$ rounds, independently execute $K$ steps of SGD starting from a common iterate, and then communicate and average their iterates to obtain the common iterate from which the next round begins.  Overall, each machine computes $T=KR$ stochastic gradients and executes $KR$ SGD steps locally, for a total of $N=KRM$ overall stochastic gradients computed (and so $N=KRM$ samples used), with $R$ rounds of communication (every $K$ steps of computation).

Given the appeal and usage of local SGD, there is significant value in understanding its performance and limitations theoretically, and in comparing it to other alternatives and baselines {\em that have the same computation and communication structure}.  That is, other methods that are distributed across $M$ machines and compute $K$ gradients per round of communication for $R$ rounds, for a total of $T=KR$ gradients per machine and $R$ communication steps. \removed{\natinote{I don't think this sentence is needed, and it seems to agressive, plus the word "fair" is always loaded}Any other comparison would be unfair.}
This structure can also be formalized through the graph oracle model of \citet[see also Section \ref{sec:setting}]{woodworth2018graph}.

So, how does local SGD compare to other algorithms with the same computation and communication structure? Is local SGD (or perhaps an accelerated variant) optimal in the same way that (accelerated) SGD is optimal in the sequential setting?  Is it better than baselines?  

A natural alternative and baseline is minibatch SGD \citep{dekel2012optimal,cotter2011better,shamir2014distributed} -- a simple method for which we have a complete and tight theoretical understanding.  Within the same computation and communication structure, minibatch SGD can be implemented as follows: Each round, calculate the $K$ stochastic gradient estimates (at the current iterate) on each machine, and then average all $KM$ estimates to obtain a single gradient estimate.  That is, we can implement minibatch SGD that takes $R$ stochastic gradient steps, with each step using a minibatch of size $KM$---this is the fair and correct minibatch SGD to compare to, and when we refer to ``minibatch SGD'' we refer to this implementation ($R$ steps with minibatch size $KM$).

Local SGD seems intuitively better than minibatch SGD, since even when the workers are not communicating, they are making progress towards the optimum. In particular, local SGD performs $K$ times more updates over the course of optimization, and can be thought of as computing gradients at less ``stale'' and more ``updated'' iterates.  For this reason, it has been argued that local SGD is at least as good as minibatch SGD, especially in convex settings where averaging iterates cannot hurt you.  But can we capture this advantage theoretically to understand how and when local SGD is better than minibatch SGD?  Or even just establish that local SGD is at least as good?

A string of recent papers have attempted to analyze local SGD for convex objectives, \citep[e.g.][]{stich2018local,stich2019error,khaled2019better,dieuleveut2019communication}. However, a satisfying analysis has so far proven elusive. In fact, every analysis that we are aware of for local SGD in the general convex (or strongly convex) case with a typical noise scaling (e.g.~as arising from supervised learning) not only does not improve over minibatch SGD, but is actually strictly dominated by minibatch SGD\removed{\footnote{\citet{khaled2019better} are able to show improvement over minibatch SGD but only in a situation where the noise in the gradients is extremely small. See Section \ref{sec:setting} and Appendix \ref{app:khaled-discussion} for more details.\natinote{This footnote is not needed---if someone is concerned, all the details are in the Section, and the statement here is correct as is}}}! But is this just a deficiency of these analyses, or is local SGD actually not better, and perhaps worse, than minibatch SGD?  In this paper, we show that the answer to this question is ``sometimes.'' There is a regime in which local SGD indeed matches or improves upon minibatch SGD, but perhaps surprisingly, there is also a regime in which local SGD really is strictly worse than minibatch SGD.
\vspace{-0.8em}
\subsubsection*{Our contributions}
In Section \ref{sec:quadratic}, we start with the special case of \textbf{quadratic} objectives and show that, at least in this case, {\bf local SGD is strictly better than minibatch SGD} in the worst case, and that an accelerated variant is even {\bf minimax optimal}.

We then turn to general {\bf convex objectives}.   In Section \ref{sec:our-upper-bound} we prove the {\bf first error upper bound on the performance of local SGD which is \emph{not} dominated by minibatch SGD's} upper bound with a typical noise scaling. In doing so, we identify a regime (where $M$ is large and $K\gtrsim R$) in which local SGD performs strictly better than minibatch in the worst case. However, our upper bound does not show that local SGD is \emph{always} as good or better than minibatch SGD. In Section \ref{sec:non-quadratic}, we show that this is not just a failure of our analysis. We prove a {\bf lower bound on the worst-case error of local SGD that is \emph{higher} than the worst-case error of minibatch SGD in a certain regime!}  We demonstrate this behaviour empirically, using a logistic regression problem where local SGD indeed behaves much worse than mini-batch SGD in the theoretically-predicted problematic regime.

Thus, while local SGD is frequently better than minibatch SGD---and we can now see this both in theory and in practice \citep[see experiments by e.g.][]{zhang2016parallel,lin2018don,Zhou2018:Kaveraging}---our work identifies regimes in which users should be wary of using local SGD without considering alternatives like minibatch SGD, and might want to seek alternative methods that combine the best of both, and attain optimal performance in all regimes.

\vspace{-.5em}
\section{Preliminaries}\label{sec:setting}
We consider the stochastic convex optimization problem:
\begin{equation}\label{eq:general-stochastic-objective}
    \min_{x \in \R^d} F(x) := \underset{z\sim\mc{D}}{\E}\brk*{f(x;z)}\,.
\end{equation}
We will study distributed first-order algorithms that compute stochastic gradient estimates at a point $x \in \R^d$ via $\nabla f(x;z)$ based on indpendent samples $z \sim \mc{D}$. Our focus is on objectives $F$ that are $H$-smooth, either (general) convex or $\lambda$-strongly convex\footnote{An $H$-smooth and $\lambda$-strongly convex function satisfies $\frac{\lambda}{2}\nrm*{x-y}^2 \leq F(y) - F(x) - \tri*{\nabla F(x), y-x} \leq \frac{H}{2}\nrm*{x-y}^2$. We allow $\lambda = 0$ in which case $F$ is general convex.}, with a minimizer $x^* \in \argmin_x F(x)$ with $\nrm{x^*} \leq B$. We consider $\nabla f$ which has uniformly bounded variance, i.e.~$\sup_x\E_{z\sim\mc{D}}\nrm*{\nabla f(x;z) - \nabla F(x)}^2 \leq \sigma^2$. We use $\mc{F}(H,\lambda,B,\sigma^2)$ to refer to the set of all pairs $(f,\mc{D})$ which satisfy these properties. All of the analysis in this paper can be done either for general convex or strongly convex functions, and we prove all of our results for both cases. For conciseness and clarity, when discussing the results in the main text, we will focus on the general convex case. However, the picture in the strongly convex case is mostly the same. 

An important instance of \eqref{eq:general-stochastic-objective} is a supervised learning problem where $f(x;z) = \ell\prn*{\tri*{x, \phi(z)}, \textit{label}(z)}$ is the loss on a single sample.  When $\abs{\ell'}, \abs{\ell''} \leq 1$ (referring to derivatives w.r.t.~the first argument), then $H \leq \abs{\ell''} \norm{\phi(z)}^2 \leq \norm{\phi(z)}^2$ and also $\sigma^2 \leq \norm{\nabla f}^2 \leq \abs{\ell'}^2 \norm{\phi(z)}^2 \leq \norm{\phi(z)}^2$. Thus, assuming that the upper bounds on $\ell',\ell''$ are comparable, the relative scaling of parameters we consider as most ``natural'' is $H \approx \sigma^2$.

For simplicity, we consider initializing all algorithms at zero. Then, Local SGD with $M$ machines, $K$ stochastic gradients per round, and $R$ rounds of communication calculates its $t$th iterate on the $m$th machine for $t \in [KR]$ via \vspace{-0.5em}
\begin{equation}
\hspace{-1.5mm}
x_{t}^m = \begin{cases} x_{t-1}^m - \eta \nabla f(x_{t-1}^m; z_{t-1}^m) & K \not|\ t \\ \frac{1}{M}\sum_{m'=1}^M x_{t-1}^{m'} - \eta \nabla f(x_{t-1}^{m'}; z_{t-1}^{m'}) & K\ |\ t\end{cases}
\end{equation}
where $z_t^m\sim\mc{D}$ i.i.d., and $K\ |\ t$  refers to $K$ dividing $t$. For each $r \in [R]$, minibatch SGD calculates its $r$th iterate via
\vspace{-0.5em}
\begin{equation}\label{eq:generic-minibatch-sgd} 
x_{r} = x_{r-1} - \frac{\eta}{MK} \sum_{i=1}^{MK} \nabla f(x_{r-1}; z_{r-1}^i)
\end{equation}
We also introduce another strawman baseline, which we will refer to as ``thumb-twiddling'' SGD. In thumb-twiddling SGD, each machine computes just one (rather than $K$) stochastic gradients per round of communication and ``twiddles its thumbs'' for the remaining $K-1$ computational steps, resulting in $R$ minibatch SGD steps, but with a minibatch size of only $M$ (instead of $KM$, i.e.~as if we used $K=1$).  This is a silly algorithm that is clearly strictly worse than minibatch SGD, and we would certainly expect any reasonable algorithm to beat it.  But as we shall see, previous work has actually struggled to show that local SGD even matches, let alone beats, thumb-twiddling SGD. In fact, we will show in Section \ref{sec:non-quadratic} that, in certain regimes, local SGD truly is {\em worse} than thumb-twiddling.

For a particular algorithm $\mathsf{A}$, we define its worst-case performance with respect to $\mc{F}(H,\lambda,B,\sigma^2)$ as:
\vspace{-0.5em}
\begin{equation}
\epsilon_{\mathsf{A}} = \max_{(f,\mc{D}) \in \mc{F}(H,\lambda,B,\sigma^2)} F(\hat{x}_{\mathsf{A}}) - F(x^*)
\end{equation}
The worst-case performance of minibatch SGD for general convex objectives is tightly understood \citep{nemirovskyyudin1983,dekel2012optimal}:
\vspace{-0.5em}
\begin{equation}
\epsilon_{\textrm{MB-SGD}} = \Theta\prn*{\frac{HB^2}{R} + \frac{\sigma B}{\sqrt{MKR}}}.\label{eq:minibatch_convergence}
\end{equation}

In order to know if an algorithm like local or minibatch SGD is ``optimal'' in the worst case requires understanding the minimax error, i.e.~the best error that any algorithm with the requisite computation and communication structure can guarantee in the worst case. This requires formalizing the set of allowable algorithms. One possible formalization is the graph oracle model of \citet{woodworth2018graph} which focuses on the dependence structure between different stochastic gradient computations resulting from the communication pattern. Using this method, \citeauthor{woodworth2018graph} prove lower bounds which are applicable to our setting. Minibatch SGD does not match these lower bounds (nor does accelerated minibatch SGD, see \citet{cotter2011better}), but these lower bounds are not known to be tight, so the minimax complexity and minimax optimal algorithm are not yet known.

\begin{table}[!ht]
\caption{Comparison of existing analyses of Local SGD for general convex functions, with constant factors and low-order terms (in the natural scaling $H \approx \sigma^2$)  omitted. We applied existing upper bounds as optimistically as possible, e.g.~making additional assumptions where necessary to apply the guarantee to our setting, and our derivations are explained in Appendix \ref{app:comparison-of-existing-work}. The bolded term is the one which compares least favorably against minibatch SGD. Analogous rates for strongly convex functions are given in Appendix \ref{app:comparison-of-existing-work}.}
\vspace{0.2em}
\label{tab:comparison}
    \renewcommand{\arraystretch}{1.5}%
       \begin{minipage}{\linewidth}
       \centering
       \begin{tabular}{ p{.31\linewidth} p{.3\linewidth}  } 
        \toprule  
        Minibatch SGD & $\frac{HB^2}{R} + \frac{\sigma B}{\sqrt{MKR}}$\\
        \cmidrule{2-2}
        Thumb-twiddling SGD & $\frac{HB^2}{R} + \frac{\sigma B}{\sqrt{MR}}$\\
        \midrule 
        \citet{stich2018local} &
        $ \mathbf{\frac{HB^2}{R^{2/3}}} + \frac{HB^2}{(KR)^{3/5}} + \frac{\sigma B}{\sqrt{MKR}}$\\
        \cmidrule{2-2}
        \citet{stich2019error} & $\mathbf{\frac{HB^2 M}{R}} + \frac{\sigma B}{\sqrt{MKR}}$\\
        \cmidrule{2-2}
        \citet{khaled2019better}\footnote{This upper bound applies only when $M \leq KR$. It also requires smoothness of each $f(x;z)$ individually, i.e.~not just $F$.} & $\mathbf{\frac{\sigma^2 M}{H R}} + \frac{H^2B^2 + \sigma^2}{H\sqrt{MKR}}$ \\
        \midrule 
        Our upper bound (Section \ref{sec:our-upper-bound}) & $\mathbf{\frac{\prn*{H\sigma^2B^4}^{1/3}}{(\sqrt{K}R)^{2/3}}} + \frac{HB^2}{KR} + \frac{\sigma B}{\sqrt{MKR}}$\\
        \cmidrule{2-2} 
        Our lower bound (Section \ref{sec:non-quadratic})  & $\mathbf{\frac{\prn*{H\sigma^2B^4}^{1/3}}{(KR)^{2/3}}} + \frac{\sigma B}{\sqrt{MKR}} $\\
        \bottomrule
    \end{tabular}
\end{minipage}
\end{table}
\vspace{-0.5em}

\paragraph{Existing analysis of local SGD}
Table~\ref{tab:comparison} summarizes the best existing analyses of local SGD that we are aware of that can be applied to our setting.  We present the upper bounds as they would apply in our setting, and after optimizing over the stepsize and other parameters.  A detailed derivation of these upper bounds from the explicitly-stated theorems in other papers is provided in Appendix \ref{app:comparison-of-existing-work}.  
As we can see from the table, in the natural scaling $H = \sigma^2$, every previous upper bound is strictly dominated by minibatch SGD. Worse, these upper bounds can even be worse than even thumb-twiddling SGD when $M\gg R$ (although they are sometimes better). In particular, the first term of each previous upper bound (in terms of $M,K,R$) is never better than $R^{-1}$ (the optimization term of minibatch and thumb-twiddling SGD), and can be much worse. 

We should note that in an extremely low noise regime $\sigma^2 \leq H^2B^2\min\crl{\frac{1}{M}, \frac{K}{R}}$, the bound of \citet{khaled2019better} can sometimes improve over minibatch SGD. However, this only happens when $KR$ steps of sequential SGD is better than minibatch SGD---i.e.~when you are better off ignoring $M-1$ of the machines and just doing serial SGD on a single machine (such an approach would have error $\frac{HB^2}{KR} + \frac{\sigma B}{\sqrt{KR}}$). This is a trivial regime in which every update for any of these algorithms is essentially an exact gradient descent step, thus there is no need for parallelism in the first place. See Appendix \ref{app:khaled-discussion} for further details. The upper bound we develop in Section \ref{sec:our-upper-bound}, in contrast, dominates their guarantee and shows an improvement over minibatch that \emph{cannot} be achieved on a single machine (i.e.~without leveraging any parallelism). Furthermore, this improvement can occur even in the natural scaling $H = \sigma^2$ and even when minibatch SGD is better than serial SGD on one machine.

We emphasize that Table \ref{tab:comparison} lists the guarantees specialized to our setting---some of the bounds are presented under slightly weaker assumptions, or with a more detailed dependence on the noise: \citet{stich2019error,haddadpour2019local} analyze local SGD assuming not-quite-convexity; 
and \citet{wang2018cooperative,dieuleveut2019communication} derive guarantees under both multiplicative and additive bounds on the noise.   \citet{dieuleveut2019communication} analyze local SGD with the additional assumption of a bounded third derivative, but even with this assumption do not improve over mini-batch SGD. Numerous works study local SGD in the non-convex setting \citep[see e.g.][]{Zhou2018:Kaveraging, yu2019parallel, wang2017memory, stich2019error, haddadpour2019trading}. Although their bounds would apply in our convex setting, due to the much weaker assumptions they are understandably much worse than minibatch SGD.  There is also a large body of work studying the special case $R=1$, i.e.~where the iterates are averaged just one time at the end \citep{zinkevich2010parallelized, zhang2012communication, li2014efficient, rosenblatt2016optimality, godichon2017rates, jain2017parallelizing}. However, these analyses do not easily extend to multiple rounds, and the $R=1$ constraint can provably harm performance  \citep[see][]{shamir2014communication}. Finally, local SGD has been studied with heterogeneous data, i.e.~where each machine receives stochastic gradients from different distributions---see \citet[Sec. 3.2]{kairouz2019advances} a recent survey.

\paragraph{An Alternative Viewpoint: Reducing Communication}
In this work, we focus on understanding the best achievable error for a given $M$, $K$, and $R$. However, one might also want to know to what extent it is possible to reduce communication without paying for it. Concretely, fix $T = KR$, and consider as a baseline an algorithm which computes $T$ stochastic gradients on each machine sequentially, but is allowed to communicate after every step. We can then ask to what extent we can compete against this baseline while using less communication. One way to do this is to use Local SGD, which reduces communcation by a factor of $K$. However, the amount by which we can reduce communcation using Local SGD is easily determined once we know the error of Local SGD for each fixed $K$. Therefore, this viewpoint of reducing communcation is essentially equivalent to the one we take.

\section{Good News: Quadratic Objectives}\label{sec:quadratic}
As we have seen, existing analyses of local SGD are no better than that of minibatch SGD. In the special case where $F$ is quadratic, we will now show that not only is local SGD \emph{sometimes} as good as minibatch SGD, but it is \emph{always} as good as minibatch SGD, and sometimes better. In fact, an accelerated variant of local SGD is minimax optimal for quadratic objectives. More generally, we show that the local SGD anologue for a large family of serial first-order optimization algorithms enjoys an error guarantee which depends only on the product $KR$ and not on $K$ or $R$ individually. In particular, we consider the following family of linear update algorithms:
\begin{definition}[Linear update algorithm]\label{def:linear-update-algorithm}
We say that a first-order optimization algorithm is a linear update algorithm if, for fixed linear functions $\mc{L}^{(t)}_1,\mc{L}^{(t)}_2$, the algorithm generates its $t+1$st iterate according to 
\begin{equation}\label{eq:local-A-update}
x_{t+1} = \mc{L}^{(t)}_2\prn*{x_1,\dots,x_t,\nabla f\prn*{\mc{L}^{(t)}_1\prn*{x_1,\dots,x_t};z_t}}
\end{equation}
\end{definition}
This family captures many standard first-order methods including SGD, which corresponds to the linear mappings $\mc{L}^{(t)}_1\prn*{x_1,\dots,x_t} = x_t$ and $x_{t+1} = x_t - \eta_t \nabla f(x_t;z_t)$. Another notable algorithm in this class is AC-SA \citep{ghadimi2013optimal}, an accelerated variant of SGD which also has linear updates. Some important non-examples, however, are adaptive gradient methods like AdaGrad \citep{mcmahan2010adaptive,duchi2011adaptive}---these have linear updates, but the linear functions are data-dependent.

For a linear update algorithm $\mc{A}$, we will use local-$\mc{A}$ to denote the local SGD analogue with $\mc{A}$ replacing SGD. 
That is, during each round of communication, each machine independently executes $K$ iterations of $\mc{A}$ and then the $M$ resulting iterates are averaged. For quadratic objectives, we show that this approach inherits the guarantee of $\mc{A}$ with the benefit of variance reduction:
\begin{restatable}{theorem}{quadraticthm}\label{thm:linear-update-alg-quadratics}
Let $\mc{A}$ be a linear update algorithm which, when executed for $T$ iterations on any quadratic $(f,\mc{D})\in\mc{F}(H,\lambda,B,\sigma^2)$, guarantees $\E F(x_T) - F^* \leq \epsilon(T, \sigma^2)$. Then, local-$\mc{A}$'s averaged final iterate $\bar{x}_{KR} = \frac{1}{M}\sum_{m=1}^M x_{KR}^m$ will satisfy $\E F(\bar{x}_{KR}) - F^* \leq \epsilon\prn{KR, \frac{\sigma^2}{M}}$.
\end{restatable}
We prove this in Appendix \ref{app:quadratic} by showing that the average iterate $\bar{x}_t$ is updated according to $\mc{A}$---even in the middle of rounds of communication when $\bar{x}_t$ is not explicitly computed. In particular, we first show that
\begin{equation}
\bar{x}_{t+1}
= \mc{L}_2^{(t)}\bigg(\bar{x}_1,\dots,\bar{x}_t,
\frac{1}{M}\sum_{m'=1}^M \nabla f\prn*{\mc{L}_1^{(t)}\prn*{x_1^{m'},\dots,x_t^{m'}};z_t^{m'}}\bigg)
\end{equation}
Then, by the linearity of $\nabla F$ and $\mc{L}_1^{(t)}$, we prove
\begin{equation}
\E\brk*{\frac{1}{M}\sum_{m'=1}^M \nabla f\prn*{\mc{L}_1^{(t)}\prn*{x_1^{m'},\dots,x_t^{m'}};z_t^{m'}}} 
= \nabla F\prn*{\mc{L}_1^{(t)}\prn*{\bar{x}_1,\dots,\bar{x}_t}}
\end{equation}
and its variance is reduced to $\frac{\sigma^2}{M}$. Therefore, $\mc{A}$'s guarantee carries over while still benefitting from the lower variance.

To rephrase Theorem \ref{thm:linear-update-alg-quadratics}, on quadratic objectives, local-$\mc{A}$ is in some sense equivalent to $KR$ iterations of $\mc{A}$ with the gradient variance reduced by a factor of $M$. Furthermore, this guarantee depends only on the product $KR$, and not on $K$ or $R$ individually. Thus, averaging the $T$th iterate of $M$ independent executions of $\mc{A}$, sometimes called ``one-shot averaging,'' enjoys the same error upper bound as $T$ iterations of size-$M$ minibatch-$\mc{A}$.

Nevertheless, it is important to highlight the boundaries of Theorem \ref{thm:linear-update-alg-quadratics}. Firstly, $\mc{A}$'s error guarantee $\epsilon(T,\sigma^2)$ must not rely on any particular structure of the stochastic gradients themselves, as this structure might not hold for the implicit updates of local-$\mc{A}$. 
Furthermore, even if some structure of the stochastic gradients \emph{is} maintained for local-$\mc{A}$, the particular iterates generated by local-$\mc{A}$ will generally vary with $K$ and $R$ (even holding $KR$ constant). Thus, Theorem \ref{thm:linear-update-alg-quadratics} does \emph{not} guarantee that local-$\mc{A}$ with two different values of $K$ and $R$ would perform the same on any particular instance. We have merely proven matching upper bounds on their worst-case performance.

We apply Theorem \ref{thm:linear-update-alg-quadratics} to yield error upper bounds for local-SGD and local-AC-SA (based on the AC-SA algorithm of \citet{ghadimi2013optimal}) which is minimax optimal:
\begin{restatable}{corollary}{quadraticcorollary}\label{cor:quadratic-localSGD-ACSA}
For any quadratic $(f,\mc{D}) \in \mc{F}(H,\lambda=0,B,\sigma^2)$, there are constants $c_1$ and $c_2$ such that local-SGD returns a point $\hat{x}$ such that
\[
\E F(\hat{x}) - F^* \leq c_1\prn*{\frac{HB^2}{KR} + \frac{\sigma B}{\sqrt{MKR}}}~,
\]
and local-AC-SA returns a point $\tilde{x}$ such that
\[
\E F(\tilde{x}) - F^* \leq c_2\prn*{\frac{HB^2}{K^2R^2} + \frac{\sigma B}{\sqrt{MKR}}}~.
\]
In particular, local-AC-SA is minimax optimal for quadratic objectives.
\end{restatable}

Comparing the bound above for local SGD with the bound for minibatch SGD \eqref{eq:minibatch_convergence}, we see that the local SGD bound is strictly better, due to the first term scaling as $(KR)^{-1}$ as opposed to $R^{-1}$. We note that minibatch SGD can also be accelerated \citep{cotter2011better}, leading to a bound with better dependence on $R$, but this is again outmatched by the bound for the (accelerated) local-AC-SA algorithm above. A similar, improved bound can also be proven when the objective is a strongly convex quadratic.

\paragraph{Prior Work in the Quadratic Setting} Local SGD and related methods have been previously analyzed for quadratic objectives, but in slightly different settings. \citet{jain2017parallelizing} study a similar setting and analyze our ``minibatch SGD'' for $M = 1$ and fixed $KR$, but varying $K$ and $R$. They show that when $K$ is sufficiently small relative to $R$, then minibatch SGD can compete with $KR$ steps of serial SGD. They also show that for fixed $M > 1$ and $bT$, when $b$ is sufficiently small then the average of $M$ independent runs of minibatch SGD with $T$ steps and minibatch size $b$ can compete with $T$ steps of minibatch SGD with minibatch size $Mb$. These results are qualitatively similar to ours, but they analyze a specific algorithm while we are able to provide a guarantee for a broader class of algorithms. 
\citet{dieuleveut2019communication} analyze local SGD on quadratic objectives and show a result analogous to our Theorem \ref{thm:linear-update-alg-quadratics}. However, their result only holds when $M$ is sufficiently small relative to $K$ and $R$. 
Finally, there is a literature on ``one-shot-averaging'' for quadratic objectives, which corresponds to an extreme where the outputs of an algorithm applied to several different training sets are averaged, \citep[e.g.][]{zhang2013divide,zhang2013communication}. These results also highlight similar phenomena, but they do not apply as broadly as Theorem \ref{thm:linear-update-alg-quadratics} and they do not provide as much insight into local SGD specifically.

\vspace{-0.5em}
\section{More Good News: General Convex Objectives}\label{sec:our-upper-bound}
In this section, we present the first analysis of local SGD for general convex objectives that is not dominated by minibatch SGD. For the first time, we can identify a regime of \mkandr~in which local SGD provably performs better than minibatch SGD in the worst case. Furthermore, our analysis dominates all previous upper bounds.
\begin{restatable}{theorem}{ourlocalsgdbound}\label{thm:our-local-sgd-bound}
Let $(f,\mc{D}) \in \mc{F}(H,\lambda,B,\sigma^2)$. When $\lambda = 0$, an appropriate average of the iterates of Local SGD with an optimally tuned constant stepsize satisfies for a universal constant $c$
\begin{align*}
\E\brk*{F(\hat{x}) - F(x^*)} 
\leq c\cdot\min\bigg\{& \frac{HB^2}{KR} + \frac{\sigma B}{\sqrt{MKR}} + \frac{\prn*{H\sigma^2B^4}^{\frac{1}{3}}}{K^{1/3}R^{2/3}},\ 
\frac{HB^2}{KR} + \frac{\sigma B}{\sqrt{KR}} \bigg\}
\end{align*}
If $\lambda > 0$, then an appropriate average of the iterates of Local SGD with decaying stepsizes satisfies for a universal constant $c$
\begin{align*}
\E\brk*{F(\hat{x}) - F(x^*)}
\leq c\cdot\min\bigg\{& HB^2\exp\prn*{-\frac{\lambda KR}{4H}}+ \frac{\sigma^2}{\lambda MKR} + \frac{H\sigma^2\log\prn*{9 + \frac{\lambda KR}{H}}}{\lambda^2 KR^2}, \nonumber\\
&\ HB^2\exp\prn*{-\frac{\lambda KR}{4H}} + \frac{\sigma^2}{\lambda KR}\bigg\}.
\end{align*}
\end{restatable}
This is proven in Appendix \ref{app:upper-bound}. We use a similar approach as \citet{stich2018local}, who analyzes the behavior of the averaged iterate $\bar{x}_t = \frac{1}{M}\sum_{m=1}^M x_t^m$, even when it is not explicitly computed. They show, in particular, that the averaged iterate evolves almost according to size-$M$-minibatch SGD updates, up to a term proportional to the dispersion of the individual machines' iterates $\frac{1}{M}\sum_{m=1}^M \nrm{\bar{x}_t - x_t^m}^2$. \citeauthor{stich2018local} bounds this with $O(\eta_t^2 K^2 \sigma^2)$, but this bound is too pessimistic---in particular, it holds even if the gradients are replaced by arbitrary vectors of norm $\sigma$. In Lemma \ref{lem:distance-bound-between-local-iterates}, we improve this bound to $O(\eta_t^2 K \sigma^2)$ which allows for our improved guarantee.\footnote{In recent work, \citet{stich2019error} present a new analysis of local-SGD which, in the general convex case is of the form $\frac{MHB^2}{R} + \frac{\sigma B}{\sqrt{MKR}}$. As stated, this is strictly worse than minibatch SGD. However, we suspect that this bound should hold \emph{for any $1 \leq M' \leq M$} because, intuitively, having more machines should not hurt you. If this is true, then optimizing their bound over $M'$ yields a similar result as Theorem \ref{thm:our-local-sgd-bound}.} Our approach resembles that of \citet{khaled2019better}, which we became aware of in the process of preparing this manuscript, however our analysis is more refined. In particular, we optimize more carefully over the stepsize so that our analysis applies for any $M$, $K$, and $R$ (rather than just $M \leq KR$) and shows an improvement over minibatch SGD in a significantly broader regime, including when $\sigma^2 \gg 0$ (see Appendix \ref{app:khaled-discussion} for additional details).

\vspace{-0.5em}
\paragraph{Comparison of our bound with minibatch SGD} 
We now compare the upper bound from Theorem \ref{thm:our-local-sgd-bound} with the guarantee of minibatch SGD.
For clarity, and in order to highlight the role of $M$, $K$, and $R$ in the convergence rate, we will compare rates for general convex objectives when $H = B = \sigma^2 = 1$, and we will also ignore numerical constants and the logarithmic factor in Theorem \ref{thm:our-local-sgd-bound}. In this setting, the worst-case error of minibatch SGD is:
\begin{equation}\label{eq:ub-comparison-mb}
\epsilon_{\textrm{MB-SGD}} = \Theta\prn*{\frac{1}{R} + \frac{1}{\sqrt{MKR}}}
\end{equation}
Our guarantee for local SGD from Theorem \ref{thm:our-local-sgd-bound} reduces to:
\begin{equation}\label{eq:ub-comparison-local}
\epsilon_{\textrm{L-SGD}} \leq O\left(\frac{1}{K^{\frac{1}{3}}R^{\frac{2}{3}}} + \frac{1}{\sqrt{MKR}}\right)
\end{equation}
These guarantees have matching statistical terms of $\frac{1}{\sqrt{MKR}}$, which cannot be improved by any first-order algorithm \citep{nemirovskyyudin1983}. Therefore, in the regime where the statistical term dominates both rates, i.e.~$M^3K\lesssim R$ and $MK\lesssim R$, both algorithms will have similar worst-case performance. When we leave this noise-dominated regime, we see that local SGD's guarantee $K^{-\frac{1}{3}}R^{-\frac{2}{3}}$ is better than minibatch SGD's $R^{-1}$ when $K \gtrsim R$ and is worse when $K \lesssim R$. 
This makes sense intuitively: minibatch SGD benefits from computing very precise gradient estimates, but pays for it by taking fewer gradient steps; conversely, each local SGD update is much noisier, but local SGD is able to make $K$ times more updates.

This establishes that for general convex objectives in the large-$M$ and large-$K$ regime, local SGD will strictly outperform minibatch SGD. However, in the large-$M$ and small-$K$ regime, we are only comparing upper bounds, so it is not clear that local SGD will in fact perform worse than minibatch SGD. Nevertheless, it raises the question of whether this is the best we can hope for from local SGD. Is local SGD truly better than minibatch SGD in some regimes but worse in others? Or, should we believe the intuitive argument suggesting that local SGD is always at least as good as minibatch SGD?
\vspace{-0.5em}
\section{Bad News: Minibatch SGD Can Outperform Local SGD}\label{sec:non-quadratic}\label{sec:lower-bound}
In Section \ref{sec:quadratic}, we saw that when the objective is quadratic, local SGD is strictly better than minibatch SGD, and enjoys an error guarantee that depends only on $KR$ and not $K$ or $R$ individually. In Section \ref{sec:our-upper-bound}, we analyzed local SGD for general convex objectives and showed that local SGD \emph{sometimes} outperforms minibatch SGD. However, we did not show that it \emph{always} does, nor that it is always even competitive with minibatch SGD. We will now show that this is not simply a failure of our analysis---in a certain regime, local SGD really is inferior (in the worst-case) to minibatch SGD, and even to thumb-twiddling SGD. We show this by constructing a simple, smooth piecewise-quadratic objective in three dimensions, on which local SGD performs poorly. We define this hard instance $(f,\mc{D}) \in \mc{F}(H,\lambda,B,\sigma^2)$ as 
\vspace{-0.5em}
\begin{equation}
f(x;z) = \frac{\lambda}{2}\prn*{x_1 - \frac{B}{\sqrt{3}}}^2 + \frac{H}{2}\prn*{x_2 - \frac{B}{\sqrt{3}}}^2 
+ \frac{H}{8}\prn*{\prn*{x_3 - \frac{B}{\sqrt{3}}}^2 + \pp{x_3 - \frac{B}{\sqrt{3}}}^2} + zx_3
\label{eq:lower-bound-construction}
\end{equation}
where $\P\brk*{z=\sigma} = \P\brk*{z=-\sigma} = \frac{1}{2}$ and $\pp{y} \equiv \max\crl{y,0}$.
\vspace{-0.5em}
\begin{restatable}{theorem}{lowerbound}\label{thm:lower-bound}
For $0 \leq \lambda \leq \frac{H}{16}$, there exists $(f,\mc{D}) \in \mc{F}(H,\lambda,B,\sigma^2)$ such that for any $K \geq 2$ and $M,R \geq 1$, local SGD initialized at $0$ with any fixed stepsize, will output a point $\hat{x}$ such that for a universal constant $c$
\begin{equation}
\E F(\hat{x}) - \min_x F(x) 
\geq c\cdot \min\crl*{\frac{H^{1/3}\sigma^{2/3}B^{4/3}}{K^{2/3}R^{2/3}}, \frac{H\sigma^2}{\lambda^2K^2R^2}, H B^2} + c\cdot \min\crl*{\frac{\sigma B}{\sqrt{MKR}}, \frac{\sigma^2}{\lambda MKR}}.
\end{equation}
\end{restatable}
We defer a detailed proof of the Theorem to Appendix \ref{app:lower-bound}. Intuitively, it relies on the fact that for non-quadratic functions, the SGD updates are no longer linear as in Section \ref{sec:quadratic}, and the local SGD dynamics introduce an additional bias term which does not depend\footnote{To see this, consider for example the univariate function $f(x;z)=x^2+[x]_+^2+zx$ where $z$ is some zero-mean bounded random variable. It is easy to verify that even if we have infinitely many machines ($M=\infty$), running local SGD for a few iterations starting from the global minimum $x=0$ of $F(x):=\E_{z}[f(x;z)]$ will generally return a point bounded away from $0$. In contrast, minibatch SGD under the same conditions will remain at $0$.} on $M$, and scales poorly with $K,R$. In fact, this phenomenon is not unique to our construction, and can be expected to exist for any ``sufficiently'' non-quadratic function. With our construction, the proof proceeds by showing that the suboptimality is large unless $x_3 \approx \frac{B}{\sqrt{3}}$ but local SGD introduces a bias which causes $x_3$ to ``drift'' in the negative direction by an amount proportional to the stepsize. On the other hand, optimizing the first term of the objective requires the stepsize to be relatively large. Combining these yields the first term of the lower bound. The second term is classical and holds even for first-order algorithms that compute $MKR$ stochastic gradients sequentially \citep{nemirovskyyudin1983}.

In order to compare this lower bound with Theorem \ref{thm:our-local-sgd-bound} and with minibatch SGD, we again consider the general convex setting with $H = B = \sigma^2 = 1$. Then, the lower bound reduces to $K^{-\frac{2}{3}}R^{-\frac{2}{3}} + \prn{MKR}^{-\frac{1}{2}}$. Comparing this to Theorem \ref{thm:our-local-sgd-bound}, we see that our upper bound is tight up to a factor of $K^{-\frac{1}{3}}$ in the optimization term. Furthermore, comparing this to the worst-case error of minibatch SGD \eqref{eq:ub-comparison-mb}, we see that local SGD is indeed worse than minibatch SGD in the worst case when $K$ is small enough relative to $R$. The cross-over point is somewhere between $K \leq \sqrt{R}$ and $K \leq R$; for smaller $K$, minibatch SGD is better than local SGD in the worst case, for larger $K$, local SGD is better in the worst case. Since the optimization terms of minibatch SGD and thumb-twiddling SGD are identical, this further indicates that local SGD is even outperformed by thumb-twiddling SGD in the small $K$ and large $M$ regime.

Finally, it is interesting to note that in the \emph{strongly convex} case (where $\lambda>0$), the gap between local GD and minibatch SGD can be even more dramatic: In that case, the optimization term of minibatch SGD scales as $\exp(-R)$ (see \citet{stich2019unified} and references therein), while our theorem implies that local SGD cannot obtain a term better than $(KR)^{-2}$. This implies an exponentially worse dependence on $R$ in that term, and a worse bound as long as $ R\gtrsim \log(K)$. 

\begin{figure*}[!tb]
\centering
\includegraphics[width=0.6\textwidth]{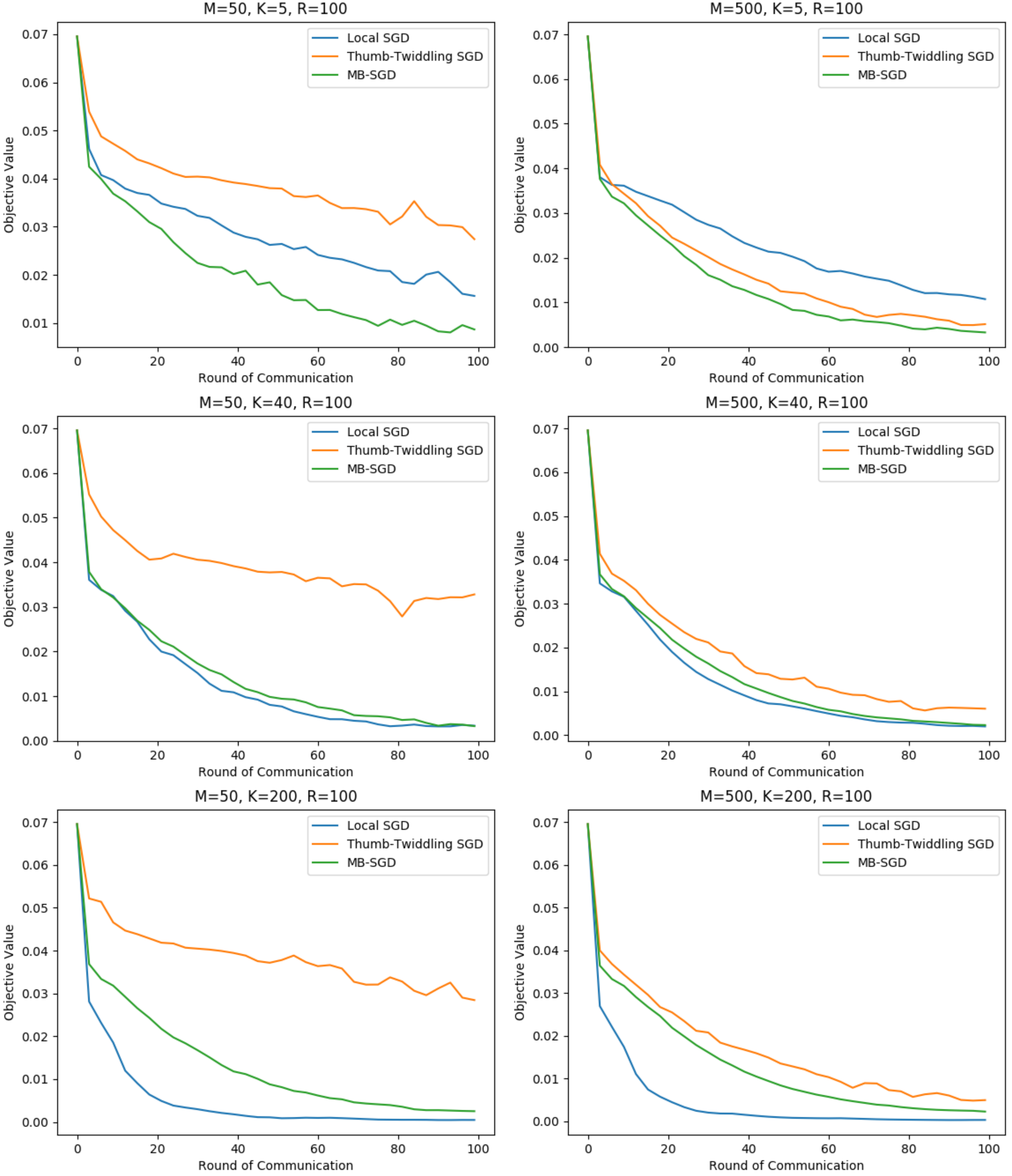}
\caption{\small We constructed a dataset of 50000 points in $\R^{25}$ with the $i$th coordinate of each point distributed independently according to a Gaussian distribution $\mc{N}(0, \frac{10}{i^2})$. The labels are generated via $\mathbb{P}\brk{y = 1\,|\, x} = \sigma(\min\crl{\tri*{w_1^*, x} + b_1^*, \tri*{w_2^*, x} + b_2^*})$ for $w_1^*, w_2^* \sim \mc{N}(0,I_{25\times 25})$ and $b_1^*, b_2^* \sim \mc{N}(0,1)$, where $\sigma(a) = 1/(1+\exp(-a))$ is the sigmoid function, i.e.~the labels correspond to an intersection of two halfspaces with label noise which increases as one approaches the decision boundary. We used each algorithm to train a linear model with a bias term to minimize the logistic loss over the 50000 points, i.e.~$f$ is the logistic loss on one sample and $\mc{D}$ is the empirical distribution over the 50000 samples. For each $M$, $K$, and algorithm, we tuned the constant stepsize to minimize the loss after $r$ rounds of communication individually for each $1 \leq r \leq R$. Let $x_{\mathsf{A},r,\eta}$ denote algorithm $\mathsf{A}$'s iterate after the $r$th round of communication when using constant stepsize $\eta$. The plotted lines are an approximation of $g_{\mathsf{A}}(r) = \min_{\eta} F(x_{\mathsf{A},r,\eta}) - F(x^*)$ for each $\mathsf{A}$ where the minimum is calculated using grid search on a log scale.} 
\label{fig:LR-experiments}
\end{figure*}

In order to prove Theorem \ref{thm:lower-bound} we constructed an artificial, but easily analyzable, situation where we could prove analytically that local SGD is worse than mini-batch. In Figure \ref{fig:LR-experiments}, we also demonstrate the behaviour empirically on a logistic regression task, by plotting the suboptimality of local SGD, minibatch SGD, and thumb-twiddling SGD iterates with optimally tuned stepsizes. As is predicted by Theorem \ref{thm:lower-bound}, we see local SGD goes from performing worse than minibatch in the small $K=5$ regime, but improving relative to the other algorithms as $K$ increases to $40$ and then $200$, when local SGD is far superior to minibatch. For each fixed $K$, increasing $M$ causes thumb-twiddling SGD to improve relative to minibatch SGD, but does not have a significant effect on local SGD, which is consistent with introducing a bias which depends on $K$ but not on $M$. This highlights that the ``problematic regime'' for local SGD is where there are few iterations per round.

\section{Future work}
In this paper, we provided the first analysis of local SGD showing improvement over minibatch SGD in a natural setting, but also demonstrated that local SGD can sometimes be worse than minibatch SGD, and is certainly not optimal.

As can be seen from Table \ref{tab:comparison}, our upper and lower bounds for local SGD are still not tight.  The first term depends on $K^{1/3}$ versus $K^{2/3}$---we believe the correct behaviour might be in between, namely $\sqrt{K}$, matching the bias of $K$-step SGD. 
The exact worst case behaviour of local SGD is therefore not yet resolved.

But beyond obtaining a precise analysis of local SGD, our paper highlights a more important challenge: we see that local SGD is definitely {\em not} optimal, and does not even always improve over minibatch SGD.  Can we suggest an optimal algorithm in this setting?  Or at least a method that combines the advantages of both local SGD and minibatch SGD and enjoys guarantees that dominate both?  Our work motivates developing such an algorithm, which might also have benefits in regimes where local SGD is already better than minibatch SGD. 

To answer this question will require new upper bounds and perhaps also new lower bounds. Looking to the analysis of local AC-SA for quadratic objectives in Corollary \ref{cor:quadratic-localSGD-ACSA}, we might hope to design an algorithm which achieves error
\begin{equation}
\E F(\hat{x}) - F(x^*) \leq O\prn*{\frac{HB^2}{(KR)^2} + \frac{\sigma B}{\sqrt{MKR}}}
\end{equation}
for general convex objectives. That is, an algorithm which combines the optimization term for $KR$ steps of accelerated gradient descent with the optimal statistical term. If this were possible, it would match the lower bound of \citet{woodworth2018graph} and therefore be optimal with respect to this communication structure.


\paragraph{Acknowledgements} 
This work is partially supported by NSF-CCF/BSF award 1718970/2016741, NSF-DMS 1547396, and a Google Faculty Research Award. BW is supported by a Google PhD Fellowship. Part of this work was done while NS was visiting Google.  Work by SS was done while visiting TTIC. 

\bibliography{bibliography}
\clearpage
\appendix
\onecolumn
\section{Comparisons Between Existing Local SGD Analyses and Minibatch SGD}\label{app:comparison-of-existing-work}
In this section, we describe the derivation of the entries in Table \ref{tab:comparison} for the cases in which it is not obvious. 
In particular, these previous analyses were stated based on different assumptions (stronger as well as weaker) which need to be reconciled with ours. Since local SGD is often analyzed in the strongly convex setting (or with weaker assumptions that are implied by strong convexity), we will make use of the following fact: If an algorithm guarantees error at most $\epsilon(\lambda)$ when applied to a $\lambda$-strongly convex function, then we can apply the algorithm to $F(x) + \frac{\lambda}{2}\nrm{x}^2$ in order to ensure error $\epsilon(\lambda) + \frac{\lambda}{2}\nrm{x^*}^2$. This applies for any $\lambda > 0$, so we can actually infer that the algorithm, in fact, guarantees error at most $\min_{\lambda > 0} \epsilon(\lambda) + \frac{\lambda}{2}\nrm{x^*}^2$.

Since our purpose is to show that these analyses are dominated by minibatch SGD, the entries in the table are, in some sense, the most optimistic interpretation of the bounds stated in the paper. For example, if error $\epsilon_1(\lambda) + \epsilon_2(\lambda)$ is guaranteed for strongly convex functions, we actually enter $\frac{1}{2}\min_{\lambda > 0} \epsilon_1(\lambda) + \frac{\lambda}{2}\nrm{x^*}^2 + \frac{1}{2}\min_{\lambda > 0} \epsilon_2(\lambda) + \frac{\lambda}{2}\nrm{x^*}^2$ into the table, which is a lower bound on the actual guarantee.

\begin{table}[]
\begin{center}
	\renewcommand{\arraystretch}{2}%
	\resizebox{\linewidth}{!}{
	\begin{tabular}{ |p{.2\linewidth}|p{.1\linewidth}|p{.7\linewidth}| } 
		\hline
		Reference & Setting & Best Convergence rate (i.e., $\E\left[{F(x^{output})-F(x^\star)}\right] \lesssim$)\\ 
		\hline 
		\multirow{2}{6em}{ \citet{stich2018local}} & SC & $\frac{\sigma^2}{\lambda MKR} + \frac{H\sigma^2}{\lambda^2 M K^2R^2} + \frac{H(H^2B^2+\sigma^2)}{\lambda^2 R^2} + \frac{H^3(H^2B^2 + \sigma^2)}{\lambda^4 K^3 R^3} + \frac{H^2B^2 + \sigma^2}{\lambda R^3}$
		\\
		\cline{2-3}
		&Non-SC& $\frac{\sigma B}{(MKR)^{1/2}} + \frac{HB^2\left(1+ (H^{-1}B^{-1}\sigma)^{2/3}\right) }{R^{2/3}} + \frac{HB^2\left(1 + (H^{-1}B^{-1}\sigma)^{2/5}\right)}{(KR)^{3/5}} + \frac{HB^2+B\sigma}{R^{3/2}}$
		\\
		\hline 
		\multirow{2}{6em}{ \citet{stich2019error}}  &SC& $HKMB^2\exp\left(-\frac{\lambda R}{10HM}\right) + \frac{\sigma^2}{\lambda MKR}$\\
		\cline{2-3}
		&Non-SC& $\frac{HMB^2}{R} + \frac{\sigma B}{\sqrt{MKR}}$\\
		\hline
		\multirow{2}{6em}{\citet{khaled2019better}} &SC& $\frac{HB^2}{K^2R^2} + \frac{H\sigma^2}{\lambda^2 MKR} + \frac{H^2\sigma^2}{\lambda^3 K R^2} $ \\
		\cline{2-3}
		&Non-SC& $\frac{H B^2}{\sqrt{KRM}} + \frac{\sigma^2}{H\sqrt{KRM}} + \frac{\sigma^2 M}{HR}$\\
		\hline
	\end{tabular}
	}
	\label{tab:full_comparison}
	\caption{Best convergence rates up to constants in previous analyses under our assumptions.}
\end{center}    
\end{table}

For reference, we restate the worst-case guarantee of minibatch SGD:
\begin{equation}\label{eq:appendix-MBSGD-error}
\epsilon_{\textrm{MB-SGD}} \asymp \frac{HB^2}{R} + \frac{\sigma B}{\sqrt{MKR}}
\end{equation}
\subsection{\citet{stich2018local}}\label{app:stich18-discussion}
The paper makes the same assumptions as us but, in addition, assumes that the stochastic gradients are uniformly bounded, i.e. $\underset{z\sim\mc{D}}{\E}\left[\norm{\nabla f(x;z)}^2\right] \leq G^2,\ \forall x$. We relax this assumption by noting the following,
\begin{align}
\underset{z\sim\mc{D}}{\E}\left[\norm{\nabla f(x;z)}^2\right] &= \underset{z\sim\mc{D}}{\E}\left[\norm{\nabla f(x;z) - \nabla f(x^\star;z) + \nabla f(x^\star;z) - \nabla F(x^\star)}^2\right] \\
&\lesssim \underset{z\sim\mc{D}}{\E}\left[\norm{\nabla f(x;z) - \nabla f(x^\star;z)}^2\right] + \underset{z\sim\mc{D}}{\E}\left[\norm{\nabla f(x^\star;z) - \nabla F(x^\star)}^2\right] \\
&\lesssim H^2\norm{x-x^\star}^2 + \sigma^2 \\
&\lesssim H^2\norm{x^\star}^2 + \sigma^2 \\
&\leq H^2B^2 + \sigma^2 \label{eq:stich18-grad-bound}
\end{align}
In the last step we make the optimistic assumption that the iterates stray no farther from $x^*$ than they were at initialization, i.e.~$\nrm{x_0 - x^*} \leq B$. This \emph{may not be true}, so this bound is optimistic. On the other hand, it is clear that one cannot generally upper bound $\underset{z\sim\mc{D}}{\E}\left[\norm{\nabla f(x;z)}^2\right]$ any tighter than this in our setting. Since our goal is anyways to show that the analysis of \citet{stich2018local} is deficient, we continue using the bound \eqref{eq:stich18-grad-bound}. This immediately gives the result for the strongly-convex setting in \cref{tab:full_comparison}. For the non-strongly setting we extend their result by optimizing each term separately as $\epsilon(\lambda) + \frac{\lambda}{2}B^2$ and ignore the constants.   

\subsection{\citet{stich2019error}}\label{app:stich19-discussion}

The paper relaxes the convexity assumption, by assuming F is $\lambda^\star$-quasi convex, i.e., $\forall x\ F(x^\star) \leq F(x) + \inner{\nabla F}{x^\star - x} + \frac{\lambda^\star}{2}\norm{x-x^\star}^2$. This condition can also hold for certain non-convex functions and is implied by $\lambda^\star$-strong convexity. Besides they assume $H$-smoothness of $F$ and multiplicative noise for the stochastic gradients, i.e., $\underset{z\sim\mc{D}}{\E}\left[\norm{\nabla f(x;z) - \nabla F(x)}^2\right] \leq N\norm{x-x^\star}^2 + \sigma_\star^2$. The latter assumption is a relaxation of the uniform upper bound on the variance of the stochastic gradients, which we have assumed. Thus to compare to their result we set $N=0$ upper bounding the stochastic variance by $\sigma^2$ and use the strong convexity constant $\lambda$ instead of $\lambda^\star$. For the non-strongly convex setting we use their rate, along with our uniform variance bound. Besides they use specific learning rate and averaging schedules to optimize their rates. Both these rates are given in \Cref{tab:full_comparison}. For the general convex setting, we believe their dependence in $M$ is poor and is improved upon by our upper bound in \Cref{sec:our-upper-bound}.

\subsection{\citet{khaled2019better}}\label{app:khaled-discussion}
The relevant analysis from \citet{khaled2019better} is given in their Corollary 2, which is their only analysis that upper bounds the error in terms of the objective function suboptimality and in the setting where each machine receives i.i.d.~stochastic gradients. Their Corollary 2 states that when $M \leq KR$, the error is bounded by\footnote{There is a typo in their statement which omits the factor of $H$ ($L$ in their notation) from the numerator of the first term.}
\begin{equation}\label{eq:khaled-guarantee}
\epsilon_{\textrm{L-SGD}} \leq \frac{HB^2}{\sqrt{MKR}} + \frac{\sigma^2}{H\sqrt{MKR}} + \frac{\sigma^2 M}{H R}
\end{equation}
In the case where $H=B=\sigma^2=1$, it is clear that this is strictly worse than minibatch SGD since $\frac{M}{R} > \frac{1}{R}$. However, consider the case of arbitrary $H$, $B$ and $\sigma^2$ and suppose \citet{khaled2019better}'s guarantee is less than $\frac{\sigma B}{\sqrt{KR}}$, in which case
\begin{equation}
\frac{HB^2}{\sqrt{MKR}} \leq \frac{\sigma B}{\sqrt{KR}} \implies M \geq \frac{H^2B^2}{\sigma^2} \implies \frac{\sigma^2 M}{H R} \geq \frac{HB^2}{R}
\end{equation}
Consequently, \eqref{eq:khaled-guarantee} is either greater than $\frac{\sigma B}{\sqrt{KR}}$ or greater than $\frac{HB^2}{R}$. This does not mean that their upper bound is worse than minibatch SGD. However, it \emph{is} worse than minibatch SGD unless $\frac{\sigma B}{\sqrt{KR}} \leq \frac{HB^2}{R}$.

If we interrogate what this regime corresponds to, we see that it is actually a trivial regime where $KR$ steps of serial SGD, which achieves error $\frac{HB^2}{KR} + \frac{\sigma B}{\sqrt{KR}} \leq \frac{HB^2}{R}$, is actually better than minibatch SGD. That is, rather than implementing minibatch SGD distributed across the $M$ machines, we are actually better off just ignoring $M-1$ of the available machines and doing serial SGD. If this is really the right thing to do, then there was never any need for parallelism in the first place, and thus there is no reason to use local SGD, which performs no better than serial SGD in this case anyways.

\section{Proofs from Section \ref{sec:quadratic}}\label{app:quadratic}
\quadraticthm*
\begin{proof}
We will show that the average of the iterates at any particular time $\bar{x}_t = \frac{1}{M}\sum_{m=1}^M x_t^m$ evolves according to $\mc{A}$ with a lower variance stochastic gradient, even though this average iterate is not explicitly computed by the algorithm at every step. It is easily confirmed from \eqref{eq:local-A-update} that
\begin{align}
\bar{x}_{t+1} &= 
\frac{1}{M}\sum_{m'=1}^M\mc{L}^{(t)}_2\prn*{x_1^{m'},\dots,x_t^{m'},\nabla f\prn*{\mc{L}^{(t)}_1\prn*{x_1^{m'},\dots,x_t^{m'}};z_t^{m'}}} \\
&= \mc{L}^{(t)}_2\prn*{\bar{x}_1,\dots,\bar{x}_t,\frac{1}{M}\sum_{m'=1}^M\nabla f\prn*{\mc{L}^{(t)}_1\prn*{x_1^{m'},\dots,x_t^{m'}};z_t^{m'}}}
\end{align}
where we used that $\mc{L}^{(t)}_2$ is linear. We will now show that $\frac{1}{M}\sum_{m'=1}^M\nabla f\prn*{\mc{L}^{(t)}_1\prn*{x_1^{m'},\dots,x_t^{m'}};z_t^{m'}}$ is an unbiased estimate of $\nabla F\prn*{\mc{L}^{(t)}_1\prn*{\bar{x}_1,\dots,\bar{x}_t}}$ with variance bounded by $\frac{\sigma^2}{M}$. Therefore, $\bar{x}_{t+1}$ is updated exactly according to $\mc{A}$ with a lower variance stochastic gradient.

By the linearity of $\mc{L}^{(t)}_1$ and $\nabla F$ 
\begin{equation}
\E\brk*{\frac{1}{M}\sum_{m'=1}^M\nabla f\prn*{\mc{L}^{(t)}_1\prn*{x_1^{m'},\dots,x_t^{m'}};z_t^{m'}}} 
= \frac{1}{M}\sum_{m'=1}^M\nabla F\prn*{\mc{L}^{(t)}_1\prn*{x_1^{m'},\dots,x_t^{m'}}} = \nabla F\prn*{\mc{L}^{(t)}_1\prn*{\bar{x}_1,\dots,\bar{x}_t}}
\end{equation}
Furthermore, since the $z_t^m$ on each machine are independent, and $\sup_x \E\nrm*{\nabla f(x;z) - \nabla F(x)}^2 \leq \sigma^2$,
\begin{multline}
\E\nrm*{\frac{1}{M}\sum_{m'=1}^M\nabla f\prn*{\mc{L}^{(t)}_1\prn*{x_1^{m'},\dots,x_t^{m'}};z_t^{m'}} - \E\brk*{\frac{1}{M}\sum_{m'=1}^M\nabla f\prn*{\mc{L}^{(t)}_1\prn*{x_1^{m'},\dots,x_t^{m'}};z_t^{m'}}}}^2\\
= \frac{1}{M^2}\sum_{m=1}^M \E\nrm*{\nabla f\prn*{\mc{L}^{(t)}_1\prn*{x_1^{m},\dots,x_t^{m}};z_t^{m}} - \nabla F\prn*{\mc{L}^{(t)}_1\prn*{x_1^{m},\dots,x_t^{m}}}}^2 \leq \frac{\sigma^2}{M}
\end{multline}
\end{proof}

\quadraticcorollary*
\begin{proof}
It is easily confirmed that SGD and AC-SA \cite{ghadimi2013optimal} are linear update algorithms, which allows us to apply Theorem \ref{thm:linear-update-alg-quadratics}. In addition, \citet{simchowitz2018randomized} shows that any randomized algorithm that accesses an \emph{deterministic} first order oracle at most $T$ times will have error at least $\frac{cHB^2}{T^2}$ in the worst case for an $H$-smooth, convex quadratic objective, for some universal constant $c$. Therefore, the first term of local-AC-SA's guarantee cannot be improved. The second term of the guarantee also cannot be improved \cite{nemirovskyyudin1983}---in fact, this term cannot be improved even by an algorithm which is allowed to make $MKR$ \emph{sequential} calls to a stochastic gradient oracle.
\end{proof}

\section{Proof of Theorem \ref{thm:our-local-sgd-bound}}\label{app:upper-bound}
Before we prove Theorem \ref{thm:our-local-sgd-bound}, we will introduce some notation. Recall that the objective is of the form $F(x) := \E_{z\sim\mc{D}}\brk*{f(x;z)}$. Let $\eta_t$ denote the stepsize used for the $t$th overall iteration. Let $x_t^m$ denote the $t$th iterate on the $m$th machine, and let $\bxt = \frac{1}{M}\sum_{m=1}^M x_t^m$ denote the averaged $t$th iterate. The vector $\bxt$ may not actually be computed by the algorithm, but it will be central to our analysis. We will use $\nabla f(x_t^m; z_t^m)$ to denote the stochastic gradient computed at $x_t^m$ by the $m$th machine at iteration $t$, and $g_t = \frac{1}{M}\sum_{m=1}^M \nabla f(x_t^m; z_t^m)$ will denote the average of the stochastic gradients computed at time $t$. Finally, let $\bgt = \frac{1}{M}\sum_{m=1}^M \nabla F(x_t^m)$ denote the average of the full gradients computed at the individual iterates. 
\begin{lemma}[See Lemma 3.1 \cite{stich2018local}]\label{lem:ourlemma31}
Let $F$ be $H$-smooth and $\lambda$-strongly convex, let\\ $\sup_x \E\nrm*{\nabla f(x;z) - \nabla F(x)}^2 \leq \sigma^2$, and let $\eta_t \leq \frac{1}{4H}$, then the iterates of local SGD satisfy
\[
\E\brk*{F(\bxt) - F^*} \leq \prn*{\frac{2}{\eta_t} - 2\lambda}\E\nrm*{\bxt - x^*}^2 - \frac{2}{\eta_t}\E\nrm*{\bx_{t+1} -x^*}^2 + \frac{2\eta_t\sigma^2}{M} + \frac{4H}{M}\sum_{m=1}^M \E\nrm*{\bxt - x_t^m}^2
\]
\end{lemma}
\begin{proof}
This proof is nearly identical to the proof of Lemma 3.1 due to \citet{stich2018local}, and we claim no technical innovation here. We include it in order to be self-contained. 

We begin by analyzing the distance of $\bx_{t+1}$ from the optimum. Below, expectations are taken over the all of the random variables $\crl*{z_t^m}$ which determine the iterates $\crl*{x_t^m}$.
\begin{align}
&\E\nrm*{\bx_{t+1} - x^*}^2 \nonumber\\
&= \E\nrm*{\bxt - \eta_t g_t - x^*}^2 \\
&= \E\nrm*{\bxt - x^*} + \eta_t^2\E\nrm*{\bgt}^2 + \eta_t^2\E\nrm*{g_t - \bgt}^2 - 2\eta_t\E\tri*{\bxt - x^*, \bgt} \\
&\leq \E\nrm*{\bxt - x^*} + \eta_t^2\E\nrm*{\bgt}^2 + \frac{\eta_t^2\sigma^2}{M} - \frac{2\eta_t}{M}\sum_{m=1}^M\E\tri*{\bxt - x^*, \nabla f(x_t^m;z_t^m)} \\
&= \E\nrm*{\bxt - x^*} + \eta_t^2\E\nrm*{\bgt}^2 + \frac{\eta_t^2\sigma^2}{M} - \frac{2\eta_t}{M}\sum_{m=1}^M\brk*{\E\tri*{x_t^m - x^*, \nabla F(x_t^m)} + \E\tri*{\bxt - x_t^m, \nabla F(x_t^m)}}\label{eq:lemma31-initial-bound}
\end{align}
For the second equality, we used that $\E\brk*{g_t - \bgt} = 0$; for the first inequality, we used that $\E\nrm*{g_t - \bgt}^2 = \E\nrm*{\frac{1}{M}\sum_{m=1}^M \nabla f(x_t^m;z_t^m) - \nabla F(x_t^m)}^2 \leq \frac{\sigma^2}{M}$ since the individual stochastic gradient estimates are independent; and for the final equality, we used that $z_t^m$ is independent of $\bxt$.

For any vectors $v_m$, $\nrm*{\sum_{m=1}^M v_m}^2 \leq M\sum_{m=1}^M \nrm*{v_m}^2$. In addition, for any point $x$ and $H$-smooth $F$, $\nrm*{\nabla F(x)}^2 \leq 2H(F(x) - F(x^*))$, thus
\begin{equation}
\eta_t^2\E\nrm*{\bgt}^2 \leq \eta_t^2M\sum_{m=1}^M \nrm*{\frac{1}{M}\nabla F(x_t^m)}^2 \leq \frac{2H\eta_t^2}{M}\sum_{m=1}^M F(x_t^m) - F(x^*)
\end{equation}
By the $\lambda$-strong convexity of $F$, we have that
\begin{multline}
    -\frac{2\eta_t}{M}\sum_{m=1}^M\tri*{x_t^m - x^*, \nabla F(x_t^m)} \leq -\frac{2\eta_t}{M}\sum_{m=1}^M \brk*{F(x_t^m) - F(x^*) + \frac{\lambda}{2}\nrm*{x_t^m - x^*}^2} \\
    \leq -\frac{2\eta_t}{M}\sum_{m=1}^M \brk*{F(x_t^m) - F(x^*)} - \lambda\eta_t\nrm*{\bxt - x^*}^2
\end{multline}
Finally, using the fact that for any vectors $a,b$ and any $\gamma > 0$, $2\tri*{a,b} \leq \gamma \nrm{a}^2 + \gamma^{-1}\nrm{b}^2$ we have
\begin{equation}
-2\eta_t\tri*{\bxt - x_t^m, \nabla F(x_t^m)} \leq \eta_t\gamma\nrm*{\bxt - x_t^m}^2 + \frac{\eta_t}{\gamma}\nrm*{\nabla F(x_t^m)}^2
\leq \eta_t\gamma\nrm*{\bxt - x_t^m}^2 + \frac{2H\eta_t}{\gamma}[F(x_t^m) - F(x^*)]
\end{equation}
Combining these with \eqref{eq:lemma31-initial-bound}, we conclude that for $\gamma = 2H$
\begin{align}
\E\nrm*{\bx_{t+1} - x^*}^2 
&\leq \prn*{1-\lambda\eta_t}\E\nrm*{\bxt - x^*} - \frac{2\eta_t\prn*{1-H\eta_t}}{M}\sum_{m=1}^M \E\brk*{F(x_t^m) - F(x^*)} + \frac{\eta_t^2\sigma^2}{M} \nonumber\\
&\qquad\qquad+ \frac{\eta_t}{M}\sum_{m=1}^M\brk*{2H\E\nrm*{\bxt - x_t^m}^2 + \E\brk*{F(x_t^m) - F(x^*)}} \\
&= \prn*{1-\lambda\eta_t}\E\nrm*{\bxt - x^*} - \frac{\eta_t\prn*{1-2H\eta_t}}{M}\sum_{m=1}^M \E\brk*{F(x_t^m) - F(x^*)} \nonumber\\
&\qquad\qquad+ \frac{\eta_t^2\sigma^2}{M} + \frac{2H\eta_t}{M}\sum_{m=1}^M\E\nrm*{\bxt - x_t^m}^2
\end{align}
By the convexity of $F$ and the fact that $\eta_t \leq \frac{1}{4H}$, this implies
\begin{align}
\E\nrm*{\bx_{t+1} - x^*}^2 
&\leq \prn*{1-\lambda\eta_t}\E\nrm*{\bxt - x^*} - \frac{\eta_t}{2}\E\brk*{F(\bxt) - F(x^*)} + \frac{\eta_t^2\sigma^2}{M} + \frac{2H\eta_t}{M}\sum_{m=1}^M\E\nrm*{\bxt - x_t^m}^2
\end{align}
Rearranging completes the proof.
\end{proof}

We will proceed to bound the final term in Lemma \ref{lem:ourlemma31} more tightly than was done by \citet{stich2018local}, which allows us to improve on their upper bound. To do so, we will use the following technical lemmas:

\begin{lemma}[Co-Coercivity of the Gradient]\label{lem:co-coercivity}
For any $H$-smooth and convex $F$, and any $x$, and $y$
\[
\nrm*{\nabla F(x) - \nabla F(y)}^2 \leq H\inner{\nabla F(x) - \nabla F(y)}{x - y}
\]
and
\[
\nrm*{\nabla F(x) - \nabla F(y)}^2 \leq 2H\prn*{F(x) - F(y) - \inner{\nabla F(y)}{x-y}}
\]
\end{lemma}
\begin{proof}
This proof follows closely from \cite{vandenbergheLecture}. Define the $H$-smooth, convex functions
\begin{equation}
F_x(z) = F(z) - \inner{\nabla F(x)}{z} \qquad\textrm{and}\qquad F_y(z) = F(z) - \inner{\nabla F(y)}{z} 
\end{equation}
By setting the gradients of these convex functions equal to zero, it is clear that $x$ minimizes $F_x$ and $y$ minimizes $F_y$. For any $H$-smooth and convex $F$, for any $z$, $\nrm*{\nabla F(z)}^2 \leq 2H(F(z) - \min_x F(x))$, therefore,
\begin{align}
F(y) - F(x) - \inner{\nabla F(x)}{y-x} 
&= F_x(y) - F_x(x) \\
&\geq \frac{1}{2H}\nrm*{\nabla F_x(y)}^2 \\
&= \frac{1}{2H}\nrm*{\nabla F(y) - \nabla F(x)}^2
\end{align}
Similarly, 
\begin{equation}
F(x) - F(y) - \inner{\nabla F(y)}{x-y} 
\geq \frac{1}{2H}\nrm*{\nabla F(y) - \nabla F(x)}^2
\end{equation}
This is the second claim of the Lemma, and combining these last two inequalities proves the first claim.
\end{proof}

\begin{lemma}[See Lemma 6 \cite{karimireddy2019scaffold}]\label{lem:contraction-map}
Let $F$ be any $H$-smooth and $\lambda$-strongly convex function, and let $\eta \leq \frac{1}{H}$. Then for any $x,y$
\[
\nrm*{x - \eta \nabla F(x) - y + \eta\nabla F(y)}^2 \leq \prn*{1-\lambda\eta}\nrm*{x - y}^2
\]
\end{lemma}
\begin{proof}
This Lemma and its proof are essentially identical to \citep[Lemma 6]{karimireddy2019scaffold}, we include it here in order to keep our results self-contained, and we are more explicit about the steps used.
\begin{align}
\nrm*{x - \eta \nabla F(x) - y + \eta\nabla F(y)}^2
&= \nrm*{x-y}^2 + \eta^2 \nrm*{\nabla F(x) - \nabla F(y)}^2 - 2\eta\inner{\nabla F(x) - \nabla F(y)}{x-y} \\
&\leq \nrm*{x-y}^2 + \eta^2H\inner{\nabla F(x) - \nabla F(y)}{x - y}  - 2\eta\inner{\nabla F(x) - \nabla F(y)}{x-y}
\end{align}
where the inequality follows from Lemma \ref{lem:co-coercivity}. Since $\eta H \leq 1$, we further conclude that
\begin{equation}
\nrm*{x - \eta \nabla F(x) - y + \eta\nabla F(y)}^2
\leq \nrm*{x-y}^2 - \eta\inner{\nabla F(x) - \nabla F(y)}{x-y}
\end{equation}
Finally, by the $\lambda$-strong convexity of $F$
\begin{gather}
\inner{\nabla F(x)}{x-y} \geq F(x) - F(y) + \frac{\lambda}{2}\nrm*{x-y}^2 \\
-\inner{\nabla F(y)}{x-y} \geq F(y) - F(x) + \frac{\lambda}{2}\nrm*{x-y}^2
\end{gather}
Combining these, we conclude
\begin{align}
\nrm*{x - \eta \nabla F(x) - y + \eta\nabla F(y)}^2
&\leq \nrm*{x-y}^2 - \eta\inner{\nabla F(x) - \nabla F(y)}{x-y} \\
&\leq \nrm*{x-y}^2 - \eta\lambda \nrm*{x-y}^2
\end{align}
which completes the proof.
\end{proof}

\begin{lemma}\label{lem:distance-to-average-distance-to-other}
For any $t$ and $m \neq m'$
\[
\E\nrm*{x_t^m - \bxt}^2 \leq \frac{M-1}{M}\E\nrm*{x_t^m - x_t^{m'}}^2
\]
\end{lemma}
\begin{proof}
First, we note that $x_t^1,\dots,x_t^M$ are identically distributed. Therefore,
\begin{align}
\E\nrm*{x_t^m - \bxt}^2
&= \E\nrm*{x_t^m - \frac{1}{M}\sum_{m'=1}^M x_t^{m'}}^2 \\
&= \frac{1}{M^2}\E\nrm*{\frac{1}{M}\sum_{m'\neq m} x_t^m - x_t^{m'}}^2 \\
&= \frac{1}{M^2}\brk*{\sum_{m'\neq m}\E\nrm*{x_t^m - x_t^{m'}}^2 + \sum_{m'\neq m, m''\neq m, m'\neq m''} \E\inner{x_t^m - x_t^{m'}}{x_t^m - x_t^{m''}}} \\
&\leq \frac{1}{M^2}\brk*{(M-1)\E\nrm*{x_t^m - x_t^{m'}}^2 + \sum_{m'\neq m, m''\neq m, m'\neq m''} \sqrt{\E\nrm*{x_t^m - x_t^{m'}}^2\E\nrm*{x_t^m - x_t^{m''}}^2}} \\
&= \frac{1}{M^2}\brk*{(M-1)\E\nrm*{x_t^m - x_t^{m'}}^2 + 2\binom{M-1}{2} \E\nrm*{x_t^m - x_t^{m'}}^2} \\
&= \frac{(M-1)^2}{M^2}\E\nrm*{x_t^m - x_t^{m'}}^2 \\
&\leq \frac{M-1}{M}\E\nrm*{x_t^m - x_t^{m'}}^2
\end{align}
\end{proof}

\begin{lemma}\label{lem:distance-bound-between-local-iterates}
Under the conditions of Lemma \ref{lem:ourlemma31}, with the additional condition that the sequence of stepsizes $\eta_1,\eta_2,\dots$ is non-increasing and $\eta_t \leq \frac{1}{H}$ for all $t$, for any $t$ and any $m$
\[
\E\nrm*{x_t^m - \bxt}^2 \leq \frac{2(M-1)(K-1)\eta_{t-K+1 \land 0}^2\sigma^2}{M}
\]
If $\eta_t = \frac{2}{\lambda(a+t+1)}$, then it further satisfies
\[
\E\nrm*{x_t^m - \bxt}^2 \leq \frac{2(M-1)(K-1)\eta_{t-1}^2\sigma^2}{M}
\]
\end{lemma}
\begin{proof}
By Lemma \ref{lem:distance-to-average-distance-to-other}, we can upper bound
\begin{equation}
    \E\nrm*{x_t^m - \bxt}^2 \leq \frac{M-1}{M}\E\nrm*{x_t^m - x_t^{m'}}^2
\end{equation}
for all $t$ and $m\neq m'$. In addition,
\begin{align}
\E\nrm*{x_t^m - x_t^{m'}}^2
&= \E\nrm*{x_{t-1}^m - \eta_{t-1} \nabla f(x_{t-1}^m;z_{t-1}^m) - x_{t-1}^{m'} + \eta_{t-1} \nabla f(x_{t-1}^{m'};z_{t-1}^{m'})}^2 \\
&\leq \E\nrm*{x_{t-1}^m - \eta_{t-1} \nabla F(x_{t-1}^m) - x_{t-1}^{m'} + \eta_{t-1} \nabla F(x_{t-1}^{m'})}^2 + 2\eta_{t-1}^2\sigma^2 \\
&\leq \prn*{1-\lambda\eta_{t-1}}\E\nrm*{x_{t-1}^m - x_{t-1}^{m'}}^2 + 2\eta_{t-1}^2\sigma^2
\end{align}
where for the final inequality we used Lemma \ref{lem:contraction-map} and the fact that the stepsizes are less than $\frac{1}{H}$,. Since the iterates are averaged every $K$ iterations, for each $t$, there must be a $t_0$ with $0 \leq t-t_0 \leq K-1$ such that $x_{t_0}^m = x_{t_0}^{m'}$. Therefore, we can unroll the recurrence above to conclude that
\begin{equation}
\E\nrm*{x_t^m - x_t^{m'}}^2
\leq \sum_{i=t_0}^{t-1}2\eta_i^2\sigma^2\prod_{j=i+1}^{t-1}\prn*{1-\lambda\eta_{j}} \leq 2\sigma^2\sum_{i=t_0}^{t-1}\eta_i^2
\end{equation}
where we define $\sum_{i=a}^b c_i = 0$ and $\prod_{i=a}^b c_i = 1$ for all $a > b$ and all $\crl{c_i}_{i\in\mathbb{N}}$. Therefore, for any non-increasing stepsizes, we conclude
\begin{equation}
\E\nrm*{x_t^m - \bxt}^2 \leq \frac{2\eta_{t-K+1 \land 0}^2\sigma^2(M-1)(K-1)}{M}
\end{equation}
This implies the first claim.

In the special case $\eta_t = \frac{2}{\lambda \prn*{a + t + 1}}$, we have 
\begin{align}
\E\nrm*{x_t^m - x^{m'}}^2 
&\leq 2\sigma^2 \sum_{i = t_0}^{t-1} \eta_i^2 \prod_{j=i+1}^{t-1}\prn*{1-\lambda\eta_j} \\
&= 2\sigma^2 \sum_{i = t_0}^{t-1} \eta_i^2 \prod_{j=i+1}^{t-1}\prn*{\frac{a + j - 1}{a + j + 1}} \\
&= 2\sigma^2\eta_{t-1}^2 + \frac{2\sigma^2\eta_{t-2}^2(a+t-2)}{a+t} + 2\sigma^2 \sum_{i = t_0}^{t-3} \eta_i^2 \frac{(a + i)(a + i + 1)}{(a+t-1)(a+t)} \\
&= 2\sigma^2\eta_{t-1}^2\prn*{1 + \frac{(a+t)(a+t-2)}{(a+t-1)^2} + \sum_{i = t_0}^{t-3}\frac{(a + i)(a + t)}{(a+t-1)(a+i+1)}} \\
&\leq 2\sigma^2\eta_{t-1}^2\prn*{t-t_0} \\
&\leq 2(K-1)\sigma^2\eta_{t-1}^2
\end{align}
This implies the second claim.
\end{proof}

Next, we show that Local SGD is always at least as good as $KR$ steps of sequential SGD. To do so, we use the following result from \citet{stich2019unified}:
\begin{lemma}[Lemma 3 \cite{stich2019unified}]\label{lem:stich-recurrence}
For any recurrence of the form
\[
r_{t+1} \leq (1-a\gamma_t)r_t - b\gamma_t s_t + c\gamma_t^2
\]
with $a, b > 0$, there exists a sequence $0 < \gamma_t \leq \frac{1}{d}$ and weights $w_t > 0$ such that
\[
\frac{b}{W_T}\sum_{t=0}^T \brk*{s_t w_t + a r_{t+1}} \leq 32d r_0 \exp\prn*{-\frac{aT}{2d}} + \frac{36c}{aT}
\] 
where $W_T := \sum_{t=0}^T w_t$.
\end{lemma}

We now argue that Local SGD is never worse than $KR$ steps of sequential SGD:
\begin{lemma}\label{lem:local-not-worse-than-single}
Let $(f,\mc{D}) \in \mc{F}(H,\lambda,B,\sigma^2)$. When $\lambda=0$, an appropriate average of the iterates of Local SGD with an optimally tuned constant stepsize satisfies for a universal constant $c$
\[
F(\hat{x}) - F^* \leq c\cdot \frac{HB^2}{KR} + c\cdot\frac{\sigma B}{\sqrt{KR}}
\]
In the case $\lambda > 0$, then an appropriate average of the iterates of Local SGD with decreasing stepsize $\eta_t \asymp (\lambda t)^{-1}$ satisfies for a universal constant $c$
\[
F(\hat{x}) - F^* \leq c\cdot HB^2 \exp\prn*{-\frac{\lambda KR}{4H}} + c\cdot\frac{\sigma^2}{\lambda KR}
\]
\end{lemma}
\begin{proof}
Define $T := KR$ and consider the $(t+1)$st iterate on some machine $m$, $x_{t+1}^m$. If $t+1 \mod K \neq 0$, then $x_{t+1}^m = x_t^m - \eta_t\nabla f(x_t^m;z_t^m)$. In this case, for $\eta_t \leq \frac{1}{2H}$
\begin{align}
\E\nrm*{x_{t+1}^m - x^*}^2 
&= \E\nrm*{x_t^m - \eta_t\nabla f(x_t^m;z_t^m) - x^*}^2 \\
&= \E\nrm*{x_t^m - x^*}^2 + \eta_t^2 \E\nrm*{\nabla f(x_t^m;z_t^m)}^2 - 2\eta_t\E\inner{\nabla f(x_t^m;z_t^m)}{x_t^m - x^*} \\
&\leq \E\nrm*{x_t^m - x^*}^2 + \eta_t^2\sigma^2 + \eta_t^2 \E\nrm*{\nabla F(x_t^m)}^2 - 2\eta_t\E\inner{\nabla F(x_t^m)}{x_t^m - x^*} \\
&\leq \E\nrm*{x_t^m - x^*}^2 + \eta_t^2\sigma^2 + 2H\eta_t^2 \E\brk*{F(x_t^m) - F^*} - 2\eta_t\E\brk*{F(x_t^m) - F^* + \frac{\lambda}{2}\nrm*{x_t^m - x^*}^2} \\
&= (1-\lambda\eta_t)\E\nrm*{x_t^m - x^*}^2 + \eta_t^2\sigma^2 - 2\eta_t(1-H\eta_t)\E\brk*{F(x_t^m) - F^*} \\
\implies 
\E\brk*{F(x_t^m) - F^*} 
&\leq \prn*{\frac{1}{\eta_t}-\lambda}\E\nrm*{x_t^m - x^*}^2 - \frac{1}{\eta_t}\E\nrm*{x_{t+1}^m - x^*}^2 + \eta_t\sigma^2
\end{align}
Here, for the first inequality we used the variance bound on the stochastic gradients; for the second inequality we used the $H$-smoothness and $\lambda$-strong convexity of $F$; and for the final inequality we used that $H\eta_t \leq \frac{1}{2}$ and rearranged.

If, on the other hand, $t+1 \mod K = 0$, then $x_{t+1}^m = \frac{1}{M}\sum_{m'=1}^M x_t^{m'} - \eta_t\nabla f(x_t^{m'};z_t^{m'})$. Since the local iterates on the different machines are \emph{identically distributed},
\begin{align}
\E\nrm*{x_{t+1}^m - x^*}^2 
&= \E\nrm*{\frac{1}{M}\sum_{m'=1}^M x_t^{m'} - \eta_t\nabla f(x_t^{m'};z_t^{m'}) - x^*}^2 \\
&\leq \frac{1}{M}\sum_{m'=1}^M\E\nrm*{x_t^{m'} - \eta_t\nabla f(x_t^{m'};z_t^{m'}) - x^*}^2 \\
&= \E\nrm*{x_t^m - \eta_t\nabla f(x_t^m;z_t^m) - x^*}^2
\end{align}
Where for the first inequality we used Jensen's inequality, and for the final equality we used that the local iterates are identically distributed. From here, using the same computation as above, we conclude that in either case
\begin{equation}\label{eq:serial-sgd-recurrence}
\E\brk*{F(x_t^m) - F^*} 
\leq \prn*{\frac{1}{\eta_t}-\lambda}\E\nrm*{x_t^m - x^*}^2 - \frac{1}{\eta_t}\E\nrm*{x_{t+1}^m - x^*}^2 + \eta_t\sigma^2
\end{equation}

\paragraph{Weakly Convex Case $\lambda = 0$:}
Choose a constant learning rate $\eta_t = \eta = \min\crl*{\frac{1}{2H}, \frac{B}{\sigma\sqrt{T}}}$ and define the averaged iterate
\begin{equation}
    \hat{x} = \frac{1}{MT}\sum_{m=1}^M\sum_{t=1}^T x_t^m
\end{equation}
Then, by the convexity of $F$:
\begin{align}
\E F(\hat{x}) - F^*
&\leq \frac{1}{MT}\sum_{m=1}^M\sum_{t=1}^T\E\brk*{F(x_t^m) - F^*} \\
&\leq \frac{1}{MT}\sum_{m=1}^M\sum_{t=1}^T\frac{1}{\eta}\E\nrm*{x_t^m - x^*}^2 - \frac{1}{\eta}\E\nrm*{x_{t+1}^m - x^*}^2 + \eta\sigma^2 \\
&= \frac{\nrm*{x_0 - x^*}^2}{T\eta} + \eta\sigma^2 \\
&= \max\crl*{\frac{2H\nrm*{x_0 - x^*}^2}{T}, \frac{\sigma \nrm*{x_0 - x^*}}{\sqrt{T}}} + \frac{\sigma \nrm*{x_0 - x^*}}{\sqrt{T}} \\
&\leq \frac{2H\nrm*{x_0 - x^*}^2}{T} + \frac{2\sigma \nrm*{x_0 - x^*}}{\sqrt{T}}
\end{align}

\paragraph{Strongly Convex Case $\lambda > 0$:}
Rearranging \eqref{eq:serial-sgd-recurrence}, we see that it has the same form as the recurrence analyzed in Lemma \ref{lem:stich-recurrence} with $r_t = \E\nrm*{x_t^m - x^*}^2$, $s_t = \E\brk*{F(x_t^m) - F^*}$, $a = \lambda$, $c = \sigma^2$, and $\gamma_t = \eta_t$ with the requirement that $\eta_t \leq \frac{1}{2H}$, i.e.~$d=2H$. Consequently, by Lemma \ref{lem:stich-recurrence}, we conclude that there is a sequence of stepsizes and weights $w_t$ such that
\begin{align}
\E\brk*{F\prn*{\frac{1}{M\sum_{t=0}^{KR} w_t} \sum_{m=1}^M\sum_{t=0}^{KR} w_t x_t^m} - F^*}
&\leq \frac{1}{M\sum_{t=0}^{KR} w_t} \sum_{m=1}^M\sum_{t=0}^{KR} \E\brk*{F\prn*{w_t x_t^m} - F^*} \\
&\leq 64H\E\nrm*{x_0 - x^*}^2 \exp\prn*{-\frac{\lambda KR}{4H}} + \frac{36\sigma^2}{\lambda KR}
\end{align}
The stepsizes and weights are chosen as follows:
If $KR \leq \frac{2H}{\lambda}$, then $\eta_t = \frac{1}{2H}$ and $w_t = (1-\lambda\eta)^{-t-1}$.
If $KR > \frac{2H}{\lambda}$ and $t < KR/2$, then $\eta_t = \frac{1}{2H}$ and $w_t = 0$.
If $KR > \frac{2H}{\lambda}$ and $t \geq KR/2$, then $\eta_t = \frac{2}{4H + \lambda(t - KR/2)}$ and $w_t = (4H/\lambda + t - KR/2)^2$.
This completes the proof.
\end{proof}

Finally, we prove our main analysis of Local SGD. Portions of the analysis of the strongly convex case follow closely the proof of \citep[Lemma 3]{stich2019unified}.
\ourlocalsgdbound*
\begin{proof}
We will prove the first terms in the $\min$'s in Theorem in two parts, first for the convex case $\lambda = 0$, then for the strongly convex case $\lambda > 0$. Then, we conclude by invoking Lemma \ref{lem:local-not-worse-than-single} showing that Local SGD is never worse than $KR$ steps of SGD on a single machine, which corresponds to the second terms in the $\min$'s in the Theorem statement.

\paragraph{Convex Case $\lambda = 0$:}
By Lemma \ref{lem:ourlemma31} and the first claim of Lemma \ref{lem:distance-bound-between-local-iterates}, the mean iterate satisfies
\begin{equation}
\E\brk*{F(\bxt) - F^*} \leq \frac{2}{\eta_t}\E\nrm*{\bxt - x^*}^2 - \frac{2}{\eta_t}\E\nrm*{\bx_{t+1} -x^*}^2 + \frac{2\eta_t\sigma^2}{M} + \frac{8H(M-1)(K-1)\eta_{t-K+2 \land 0}^2\sigma^2}{M}
\end{equation}
Consider a fixed stepsize $\eta_t = \eta$ which will be chosen later, and consider the average of the iterates
\begin{equation}
    \hat{x} = \frac{1}{KR}\sum_{t=1}^{KR} \bxt
\end{equation}
By the convexity of $F$,
\begin{align}
\E \brk*{F(\hat{x}) - F^*}
&\leq \frac{1}{KR}\sum_{t=1}^{KR} \E\brk*{F(\bxt) - F^*} \\
&\leq \frac{1}{KR}\sum_{t=1}^{KR} \brk*{\frac{2}{\eta}\E\nrm*{\bxt - x^*}^2 - \frac{2}{\eta}\E\nrm*{\bx_{t+1} -x^*}^2 + \frac{2\eta\sigma^2}{M} + \frac{8H(M-1)(K-1)\eta^2\sigma^2}{M}} \\
&\leq \frac{2B^2}{\eta KR} + \frac{2\eta\sigma^2}{M} + \frac{8H(M-1)(K-1)\eta^2\sigma^2}{M}
\end{align}
Choose as a stepsize 
\begin{equation}
\eta = \begin{cases}
\min\crl*{\frac{1}{4H},\ \frac{B\sqrt{M}}{\sigma\sqrt{KR}}} & K=1\textrm{ or }M=1 \\
\min\crl*{\frac{1}{4H},\ \frac{B\sqrt{M}}{\sigma\sqrt{KR}},\ \prn*{\frac{B^2}{H\sigma^2K^2 R}}^{\frac{1}{3}}} & \textrm{Otherwise }
\end{cases}
\end{equation}
Then,
\begin{align}
\E \brk*{F(\hat{x}) - F^*}
&\leq \frac{2B^2}{\eta KR} + \frac{2\eta\sigma^2}{M} + \frac{8H(M-1)(K-1)\eta^2\sigma^2}{M} \\
&\leq \max\crl*{\frac{8HB^2}{KR},\ \frac{2\sigma B}{\sqrt{MKR}},\ \frac{2\prn*{H\sigma^2B^4}^{\frac{1}{3}}}{K^{1/3}R^{2/3}}} + \frac{2\sigma B}{\sqrt{MKR}} + \frac{8\prn*{H\sigma^2B^4}^{\frac{1}{3}}}{K^{1/3}R^{2/3}} \\
&\leq \frac{8HB^2}{KR} + \frac{4\sigma B}{\sqrt{MKR}} + \frac{10\prn*{H\sigma^2B^4}^{\frac{1}{3}}}{K^{1/3}R^{2/3}}
\end{align}

\paragraph{Strongly Convex Case $\lambda > 0$:}
For the strongly convex case, following \citet{stich2019unified}'s proof of Lemma \ref{lem:stich-recurrence}, we choose stepsizes according to the following set of cases:
If $KR \leq \frac{2H}{\lambda}$, then $\eta_t = \frac{1}{4H}$ and $w_t = (1-\lambda\eta)^{-t-1}$.
If $KR > \frac{2H}{\lambda}$ and $t \leq KR/2$, then $\eta_t = \frac{1}{4H}$ and $w_t = 0$.
If $KR > \frac{2H}{\lambda}$ and $t > KR/2$, then $\eta_t = \frac{2}{8H + \lambda(t - KR/2)}$ and $w_t = (8H/\lambda + t - KR/2)$.
We note that in the second and third cases, the stepsize is either constant or equal to $\eta_t = \frac{2}{\lambda(a + t - KR/2)}$ (for $a = \frac{8H}{\lambda}$) within each individual round of communication. 

By Lemma \ref{lem:ourlemma31} and the first claim of Lemma \ref{lem:distance-bound-between-local-iterates}, during the rounds of communication for which the stepsize is constant, we have the recurrence:
\begin{equation}\label{eq:recurrent-small-t}
\E\nrm*{\bx_{t+1} - x^*}^2 \leq \prn*{1 - \lambda\eta_t}\E\nrm*{\bx_t - x^*}^2 - \frac{\eta_t}{2}\E\brk*{F(\bx_t) - F^*} + \frac{\eta_t^2\sigma^2}{M} + 4HK\eta_t^3\sigma^2
\end{equation}
On the other hand, during the rounds of communication in which the stepsize is decreasing, we have by Lemma \ref{lem:ourlemma31} and the second claim of Lemma \ref{lem:distance-bound-between-local-iterates} that:
\begin{equation}
\E\nrm*{\bx_{t+1} - x^*}^2 \leq \prn*{1 - \lambda\eta_t}\E\nrm*{\bx_t - x^*}^2 - \frac{\eta_t}{2}\E\brk*{F(\bx_t) - F^*} + \frac{\eta_t^2\sigma^2}{M} + 4HK\eta_t\eta_{t-1}^2\sigma^2
\end{equation}
Furthermore, during the rounds (i.e.~when $t > KR$) where the stepsize is decreasing, 
\begin{equation}
\eta_{t-1}^2 = \eta_t^2 \frac{\prn*{a + t - KR/2}^2}{\prn*{a - 1 + t - KR/2}^2} \leq 4\eta_t^2
\end{equation}
So, for every $t$ we conclude
\begin{equation}
\E\nrm*{\bx_{t+1} - x^*}^2 \leq \prn*{1 - \lambda\eta_t}\E\nrm*{\bx_t - x^*}^2 - \frac{\eta_t}{2}\E\brk*{F(\bx_t) - F^*} + \frac{\eta_t^2\sigma^2}{M} + 16HK\eta_t^3\sigma^2
\end{equation}

First, suppose $KR > \frac{2H}{\lambda}$, and consider the steps during which $\eta_t = \frac{1}{4H}$:
\begin{align}
\E\nrm*{\bx_{KR/2} - x^*}^2 
&\leq \prn*{1 - \frac{\lambda}{4H}}\E\nrm*{\bx_t - x^*}^2 - \frac{1}{8H}\E\brk*{F(\bx_t) - F^*} + \frac{\sigma^2}{16H^2 M} + \frac{K\sigma^2}{4H^2} \\
&\leq \prn*{1 - \frac{\lambda}{4H}}\E\nrm*{\bx_t - x^*}^2 + \frac{\sigma^2}{16H^2 M} + \frac{K\sigma^2}{4H^2} \\
&\leq \prn*{1 - \frac{\lambda}{4H}}^{KR/2}\E\nrm*{\bx_0 - x^*}^2 + \prn*{\frac{\sigma^2}{16H^2 M} + \frac{K\sigma^2}{4H^2}}\sum_{t=0}^{KR/2 - 1} \prn*{1 - \frac{\lambda}{4H}}^{t} \\
&\leq \prn*{1 - \frac{\lambda}{4H}}^{KR/2}\E\nrm*{\bx_0 - x^*}^2 + \frac{4H}{\lambda}\prn*{\frac{\sigma^2}{16H^2 M} + \frac{K\sigma^2}{4H^2}} \\
&\leq \E\nrm*{\bx_0 - x^*}^2\exp\prn*{-\frac{\lambda KR}{8H}} + \frac{\sigma^2}{4H\lambda M} + \frac{K\sigma^2}{H\lambda}\label{eq:recurrence-mess}
\end{align}

Now, consider the remaining steps. Rearranging, we have
\begin{align}
\E\brk*{F(\bx_t) - F^*}
&\leq \prn*{\frac{2}{\eta_t} - \frac{\lambda}{2}}\E\nrm*{\bx_t - x^*}^2 - \frac{2}{\eta_t}\E\nrm*{\bx_{t+1} - x^*}^2 + \frac{2\eta_t\sigma^2}{M} + 32HK\eta_t^2\sigma^2
\end{align}
So, since $\eta_t = \frac{2}{\lambda(a+t)}$ where $a = \frac{8H}{\lambda} - \frac{KR}{2}$ and $w_t = (a+t)$, we have
\begin{align}
&\frac{1}{W_T}\sum_{t=KR/2}^{KR} w_t \E\brk*{F(\bx_t) - F^*} \nonumber\\
&\leq \frac{1}{W_T}\sum_{t=KR/2}^{KR}w_t\brk*{\prn*{\frac{2}{\eta_t} - 2\lambda}\E\nrm*{\bx_t - x^*}^2 - \frac{2}{\eta_t}\E\nrm*{\bx_{t+1} - x^*}^2 + \frac{2\eta_t\sigma^2}{M} + 32HK\eta_t^2\sigma^2} \\
&= \frac{1}{W_T}\sum_{t=KR/2}^{KR}\lambda(a+t)(a+t-2)\E\nrm*{\bx_t - x^*}^2 - \lambda(a+t)^2\E\nrm*{\bx_{t+1} - x^*}^2 + \frac{2\sigma^2}{\lambda M} + \frac{32HK\eta_t\sigma^2}{\lambda} \\
&\leq \frac{1}{W_T}\sum_{t=KR/2}^{KR}\lambda(a+t-1)^2\E\nrm*{\bx_t - x^*}^2 - \lambda(a+t)^2\E\nrm*{\bx_{t+1} - x^*}^2 + \frac{2\sigma^2}{\lambda M} + \frac{32HK\eta_t\sigma^2}{\lambda} \\
&\leq \frac{\lambda(a+KR/2-1)^2}{W_T}\E\nrm*{\bx_{KR/2} - x^*}^2 + \frac{2\sigma^2 (KR/2)}{W_T\lambda M} + \frac{64HK\sigma^2}{W_T\lambda^2}\sum_{t=KR/2}^{KR}\frac{1}{a+t} \\
&= \frac{\lambda\prn*{\frac{8H}{\lambda}-1}^2}{W_T}\E\nrm*{\bx_{KR/2} - x^*}^2 + \frac{2\sigma^2 (KR/2)}{W_T\lambda M} + \frac{64HK\sigma^2}{W_T\lambda^2}\sum_{t'=1}^{KR/2}\frac{1}{\frac{8H}{\lambda}+t'} \\
&\leq \frac{64H^2}{W_T \lambda}\E\nrm*{\bx_{KR/2} - x^*}^2 + \frac{2\sigma^2 (KR/2)}{W_T\lambda M} + \frac{64HK\sigma^2}{W_T\lambda^2}\log\prn*{e + \frac{\lambda KR}{4H}}
\end{align}
Finally, we recall \eqref{eq:recurrence-mess}, $KR > \frac{2H}{\lambda}$, and note that $W_T = \sum_{t=KR/2}^{KR} a+t \geq \frac{3K^2R^2}{8} + \frac{aKR}{2} = \frac{K^2R^2}{8} + \frac{4HKR}{\lambda} \geq \frac{8H^2}{\lambda^2}$ thus
\begin{align}
&\frac{1}{W_T}\sum_{t=KR/2}^{KR} w_t \E\brk*{F(\bx_t) - F^*} \nonumber\\
&\leq \frac{64H^2}{W_T \lambda}\prn*{\E\nrm*{\bx_0 - x^*}^2\exp\prn*{-\frac{\lambda KR}{8H}} + \frac{\sigma^2}{4H\lambda M} + \frac{K\sigma^2}{H\lambda}} + \frac{2\sigma^2 (KR/2)}{W_T\lambda M} + \frac{64HK\sigma^2}{W_T\lambda^2}\log\prn*{e + \frac{\lambda KR}{4H}}\\
&\leq \frac{64H^2}{W_T \lambda}\E\nrm*{\bx_0 - x^*}^2\exp\prn*{-\frac{\lambda KR}{8H}} + \frac{16H\sigma^2}{\lambda^2 M W_T} + \frac{64HK\sigma^2}{\lambda^2 W_T} + \frac{8\sigma^2}{\lambda MKR} + \frac{512H\sigma^2}{\lambda^2KR^2}\log\prn*{e + \frac{\lambda KR}{4H}} \\
&\leq 8\lambda\E\nrm*{\bx_0 - x^*}^2\exp\prn*{-\frac{\lambda KR}{8H}} + \frac{4\sigma^2}{\lambda M KR} + \frac{512H\sigma^2}{\lambda^2KR^2} + \frac{8\sigma^2}{\lambda MKR} + \frac{512H\sigma^2}{\lambda^2KR^2}\log\prn*{e + \frac{\lambda KR}{4H}} \\
&\leq 8\lambda\E\nrm*{\bx_0 - x^*}^2\exp\prn*{-\frac{\lambda KR}{8H}} + \frac{12\sigma^2}{\lambda M KR} + \frac{512H\sigma^2}{\lambda^2KR^2}\log\prn*{9 + \frac{\lambda KR}{H}}
\end{align}
This concludes the proof for the case $KR > \frac{2H}{\lambda}$. 

If $KR \leq \frac{2H}{\lambda}$, we use the constant stepsize $\eta_t = \eta$ and weights $w_t = (1-\lambda\eta)^{-t-1}$. Rearranging \eqref{eq:recurrent-small-t} therefore gives
\begin{align}
\E\brk*{F(\bx_t) - F^*} 
&\leq \frac{2}{\eta}\prn*{1 - \lambda\eta}\E\nrm*{\bx_t - x^*}^2 - \frac{2}{\eta}\E\nrm*{\bx_{t+1} - x^*}^2 + \frac{2\eta\sigma^2}{M} + 8HK\eta^2\sigma^2
\end{align}
so
\begin{align}
&\frac{1}{W_T}\sum_{t=1}^{KR}w_t\E\brk*{F(\bx_t) - F^*} \nonumber\\
&\leq \frac{1}{W_T}\sum_{t=1}^{KR}w_t\brk*{\frac{2}{\eta}\prn*{1 - \lambda\eta}\E\nrm*{\bx_t - x^*}^2 - \frac{2}{\eta}\E\nrm*{\bx_{t+1} - x^*}^2 + \frac{2\eta\sigma^2}{M} + 8HK\eta^2\sigma^2} \\
&= \frac{1}{W_T}\sum_{t=1}^{KR}\brk*{\frac{2}{\eta}\prn*{1 - \lambda\eta}^{-t}\E\nrm*{\bx_t - x^*}^2 - \frac{2}{\eta}\prn*{1 - \lambda\eta}^{-(t+1)}\E\nrm*{\bx_{t+1} - x^*}^2} + \frac{2\eta\sigma^2}{M} + 8HK\eta^2\sigma^2 \\
&\leq \frac{2\E\nrm*{\bx_0 - x^*}^2}{\eta W_T} + \frac{2\eta\sigma^2}{M} + 8HK\eta^2\sigma^2
\end{align}
Finally, we note that $W_T \geq (1-\lambda\eta)^{-KR-1}$ so
\begin{align}
\frac{1}{W_T}\sum_{t=1}^{KR}w_t\E\brk*{F(\bx_t) - F^*}
&\leq \frac{2\E\nrm*{\bx_0 - x^*}^2}{\eta}\exp\prn*{-\lambda\eta(KR+1)} + \frac{2\eta\sigma^2}{M} + 8HK\eta^2\sigma^2
\end{align}
We also observe that $2H \geq \lambda KR$ so with $\eta = \frac{1}{4H} \leq \frac{1}{2\lambda KR}$ we have
\begin{align}
\frac{1}{W_T}\sum_{t=1}^{KR}w_t\E\brk*{F(\bx_t) - F^*} 
&\leq 8H\E\nrm*{\bx_0 - x^*}^2\exp\prn*{-\frac{\lambda KR}{4H}}+ \frac{\sigma^2}{\lambda MKR} + \frac{2H\sigma^2}{\lambda^2 KR^2}
\end{align}
\end{proof}

\section{Proofs from Section \ref{sec:non-quadratic}}\label{app:lower-bound}
Here, we will prove the lower bound in Theorem \ref{thm:lower-bound}. Recall the objective and stochastic gradient estimator for the hard instance are defined by 
\begin{equation}\label{eq:lower-bound-construction-app}
F(x) = \frac{\mu}{2}\prn*{x_1 - b}^2 + \frac{H}{2}\prn*{x_2 - b}^2 + \frac{L}{2}\prn*{\prn*{x_3 - c}^2 + \pp{x_3 - c}^2}
\end{equation}
and
\begin{equation}\label{eq:lower-bound-stochastic-gradient-app}
\nabla f(x;z) = \nabla F(x) + \begin{bmatrix}0\\0\\z\end{bmatrix} \qquad\textrm{where}\qquad \P\brk*{z=\sigma} = \P\brk*{z=-\sigma} = \frac{1}{2}
\end{equation}
Due to the structure of the objective \eqref{eq:lower-bound-construction-app}, which decomposes as a sum over three terms which each depend only on a single coordinate, the local-SGD dynamics on each coordinate of the optimization variable are independent of each other. For this reason, we are able to analyze local-SGD on each coordinate separately.

Define the $2L$-smooth and $L$-strongly convex function
\begin{equation}\label{eq:def-gl}
g_L(x) = \frac{L}{2}x^2 + \frac{L}{2}\pp{x}^2
\end{equation}
Define a stochastic gradient estimator for $g_L$ via 
\begin{equation}\label{eq:def-glprime}
    g_L'(x,z) = g_L'(x) + z
\end{equation}
for $z \sim \textrm{Uniform}(\pm\sigma)$. Observe that the third coordinate of local-SGD on $F$ evolves exactly the same as local-SGD on the univariate function $g_L$. In the next three lemmas, we analyze the behavior of local-SGD on $g_L$:

\begin{lemma}\label{lem:sgd-two-step}
Fix $L,\eta,\sigma > 0$ such that $L\eta \leq \frac{1}{2}$. Let $x_0$ denote a random initial point with $\E x_0 \leq 0$, and let $x_2 = x_0 - \eta g_L'(x_0,z_0) - \eta g_L'(x_0 - \eta g_L'(x_0,z_0), z_1)$ be the second iterate of stochastic gradient descent with fixed stepsize $\eta$ intialized at $x_0$, and let $x_3 = x_2 - \eta g_L'(x_2, z_2)$ be the third iterate. Then
\begin{gather*}
\E x_2 \leq \begin{cases}
\frac{-\eta\sigma}{48} & \E x_0 \leq \frac{-\eta\sigma}{48} \\
\frac{-\eta\sigma}{4} + \prn*{1-L\eta}\prn*{\E x_0 + \frac{\eta\sigma}{4}} & \E x_0 \in \left(\frac{-\eta\sigma}{48}, 0\right]
\end{cases} \\
\E x_3 \leq \begin{cases}
\frac{-\eta\sigma}{48} & \E x_0 \leq \frac{-\eta\sigma}{48} \\
\frac{-\eta\sigma}{4} + \prn*{1-L\eta}^2\prn*{\E x_0 + \frac{\eta\sigma}{4}} & \E x_0 \in \left(\frac{-\eta\sigma}{48}, 0\right]
\end{cases}
\end{gather*}
\end{lemma}
\begin{proof}
Consider the $2$nd iterate of SGD with fixed stepsize $\eta$:
\begin{align}
x_2 
&= x_{1} - \eta g_L'(x_{1},z_1) \\
&= (1-L\eta)x_{1} -L\eta\pp{x_{1}} - \eta z_{1} \\
&= (1-L\eta)\prn*{x_{0} - \eta g_L'(x_{0},z_0)} - L\eta\pp{x_{0} - \eta g_L'(x_{0},z_0)} - \eta z_{1} \\
&= (1-L\eta)^2x_{0} - L\eta(1-L\eta)\pp{x_{0}}  - L\eta\pp{(1-L\eta)x_{0} - L\eta\pp{x_{0}} - \eta z_{0}} - \eta(1-\eta)z_{0} - \eta z_{1}
\end{align}
Thus,
\begin{equation}
\E x_2 
= (1-L\eta)^2\E x_{0} - L\eta(1-L\eta)\E \pp{x_{0}}
- L\eta\E \pp{(1-L\eta)x_{0} - L\eta\pp{x_{0}} - \eta z_{0}} 
\end{equation}
Define $y \defeq (1-L\eta)x_{0} - L\eta\pp{x_{0}}$, then
\begin{align}
\E \pp{(1-L\eta)x_{0} - L\eta\pp{x_{0}} - \eta z_{0}}
&= \E\pp{y - \eta z_{0}} \\
&= \frac{1}{2}\E\pp{y - \eta\sigma} + \frac{1}{2}\E\pp{y + \eta\sigma} \\
&= \E\begin{cases}
y & y > \eta\sigma \\
\frac{y + \eta\sigma}{2} & \abs{y} \leq \eta\sigma \\
0 & y < -\eta\sigma
\end{cases}
\end{align}
The function 
\begin{equation}
z \mapsto \begin{cases}
z & z > \eta\sigma \\
\frac{z + \eta\sigma}{2} & \abs{z} \leq \eta\sigma \\
0 & z < -\eta\sigma
\end{cases}
\end{equation}
is convex, so by Jensen's inequality
\begin{align}
\E x_2 
&= (1-L\eta)\E y -  L\eta\E\begin{cases}
y & y > \eta\sigma \\
\frac{y + \eta\sigma}{2} & \abs{y} \leq \eta\sigma \\
0 & y < -\eta\sigma
\end{cases} \\
&\leq (1-L\eta)\E y -  L\eta\begin{cases}
\E y & \E y > \eta\sigma \\
\frac{\E y + \eta\sigma}{2} & \abs{\E y} \leq \eta\sigma \\
0 & \E y < -\eta\sigma
\end{cases} \\
&= \begin{cases}
(1-2L\eta)\E y & \E y > \eta\sigma \\
\prn*{1-\frac{3}{2}L\eta}\E y - \frac{L\eta^2\sigma}{2} & \abs{\E y} \leq \eta\sigma \\
(1-L\eta)\E y & \E y < -\eta\sigma
\end{cases} \\
&\leq \begin{cases}
(1-2L\eta)\E y & \E y > \eta\sigma \\
\prn*{1-\frac{3}{2}L\eta}\E y - \frac{L\eta^2\sigma}{2} & \abs{\E y} \leq \eta\sigma \\
\frac{-\eta\sigma}{2} & \E y < -\eta\sigma
\end{cases}\label{eq:one-step-lemma-eq1}
\end{align} 
where we used that $L\eta \leq \frac{1}{2}$ for the final inequality. Suppose $\E x_0 \leq \frac{-\eta\sigma}{48}$ which implies $\E y \leq \frac{-(1-L\eta)\eta\sigma}{48}$. Then we are in either the second or third case of \eqref{eq:one-step-lemma-eq1}. If we are in the third case then
\begin{equation}
\E x_2 \leq \frac{-\eta\sigma}{2} \leq \frac{-\eta\sigma}{48}
\end{equation}
If we are in the second case, then
\begin{align}
\E x_2 
&\leq \prn*{1-\frac{3}{2}L\eta}\E y - \frac{L\eta^2\sigma}{2} \\
&\leq \prn*{1-\frac{3}{2}L\eta}\frac{-(1-L\eta)\eta\sigma}{48} - \frac{L\eta^2\sigma}{2} \\
&= \frac{-\eta\sigma}{48} + \frac{3(1-L\eta)L\eta^2\sigma}{96} + \frac{L\eta^2\sigma}{48} - \frac{L\eta^2\sigma}{2} \\
&\leq \frac{-\eta\sigma}{48}
\end{align}
Either way, $\E x_2 \leq \frac{-\eta\sigma}{48}$. 

Suppose instead that $\E x_0 \in \left(\frac{-\eta\sigma}{48}, 0\right]$. Then,
\begin{align}
\E x_2 
&\leq \prn*{1-\frac{3}{2}L\eta}\E y - \frac{L\eta^2\sigma}{2} \\
&\leq (1-L\eta)\E x_0 - \frac{3L\eta(1-L\eta)}{2}\E x_0 - \frac{L\eta^2\sigma}{2} \\
&\leq (1-L\eta)\E x_0 + \frac{3L\eta}{2}\cdot\frac{\eta\sigma}{48} - \frac{L\eta^2\sigma}{2} \\
&\leq (1-L\eta)\E x_0 - \frac{L\eta^2\sigma}{4} \\
&= -\frac{\eta\sigma}{4} + \prn*{1-L\eta}\prn*{\E x_0 + \frac{\eta\sigma}{4}}
\end{align}
We conclude that
\begin{equation}\label{eq:one-step-lemma-x2}
\E x_2 \leq \begin{cases}
\frac{-\eta\sigma}{48} & \E x_0 \leq \frac{-\eta\sigma}{48} \\
\frac{-\eta\sigma}{4} + \prn*{1-L\eta}\prn*{\E x_0 + \frac{\eta\sigma}{4}} & \E x_0 \in \left( \frac{-\eta\sigma}{48}, 0 \right]
\end{cases}
\end{equation}

Now, consider the third iterate of SGD, $x_3$:
\begin{align}
\E x_3 
&= \E x_2 - \eta \E g'_L(x_2, z_2) \\
&= (1-L\eta)\E x_2 - L\eta\E \pp{x_2} \\
&= (1-L\eta)\E x_2 - L\eta\E \pp{\E[x_2\,|\,x_1] - \eta z_1} \\
&\leq (1-L\eta)\E x_2 - \frac{L\eta}{2}\E \pp{\E[x_2\,|\,x_1] + \eta \sigma}
\end{align}
Since $z \mapsto \pp{z}$ is convex, by Jensen's inequality
\begin{align}
\E x_3 
&\leq (1-L\eta)\E x_2 - \frac{L\eta}{2}\pp{\E x_2 + \eta\sigma} \\
&\leq \begin{cases}
\prn*{1-\frac{3L\eta}{2}}\E x_2 - \frac{L\eta^2\sigma}{2} & \E x_2 > -\eta\sigma \\
(1-L\eta)\E x_2 & \E x_2 \leq -\eta\sigma 
\end{cases} \\
&\leq \begin{cases}
\prn*{1-\frac{3L\eta}{2}}\E x_2 - \frac{L\eta^2\sigma}{2} & \E x_2 > -\eta\sigma \\
\frac{-\eta\sigma}{2} & \E x_2 \leq -\eta\sigma
\end{cases}\label{eq:one-step-lemma-eq2}
\end{align}
To complete the proof, we must show that 
\begin{equation}\label{eq:one-step-lemma-x3}
\E x_3 \leq \begin{cases}
\frac{-\eta\sigma}{48} & \E x_0 \leq \frac{-\eta\sigma}{48} \\
\frac{-\eta\sigma}{4} + \prn*{1-L\eta}^2\prn*{\E x_0 + \frac{\eta\sigma}{4}} & \E x_0 \in \left(\frac{-\eta\sigma}{48}, 0\right]
\end{cases}
\end{equation}
Returning to \eqref{eq:one-step-lemma-eq2}, note that if $\E x_2 \leq -\eta \sigma$ then $\E x_3 \leq \frac{-\eta\sigma}{2}$ implies \eqref{eq:one-step-lemma-x3}. Therefore, we only need to consider the first case of \eqref{eq:one-step-lemma-eq2}.

Suppose first that $\E x_0 \leq \frac{-\eta\sigma}{48}$, then by \eqref{eq:one-step-lemma-x2} we have $\E x_2 \leq \frac{-\eta\sigma}{48}$, thus
\begin{align}
\E x_3 
&\leq \prn*{1-\frac{3L\eta}{2}}\E x_2 - \frac{L\eta^2\sigma}{2} \\
&\leq \prn*{1-\frac{3L\eta}{2}}\frac{-\eta\sigma}{48} - \frac{L\eta^2\sigma}{2} \\
&\leq \frac{-\eta\sigma}{48}
\end{align}
If instead $\E x_0 \in \left(\frac{-\eta\sigma}{48}, 0\right]$, then by \eqref{eq:one-step-lemma-x2} we have $\E x_2 \leq \frac{-\eta\sigma}{4} + \prn*{1-L\eta}\prn*{\E x_0 + \frac{\eta\sigma}{4}}$, thus
\begin{align}
\E x_3 
&\leq \prn*{1-\frac{3L\eta}{2}}\E x_2 - \frac{L\eta^2\sigma}{2} \\
&\leq \prn*{1-\frac{3L\eta}{2}}\frac{-\eta\sigma}{4} + \prn*{1-\frac{3L\eta}{2}}\prn*{1-L\eta}\prn*{\E x_0 + \frac{\eta\sigma}{4}} - \frac{L\eta^2\sigma}{2}  \\
&\leq \frac{-\eta\sigma}{4} + \frac{3L\eta^2\sigma}{8} - \frac{L\eta^2\sigma}{2} + \prn*{1-L\eta}^2\prn*{\E x_0 + \frac{\eta\sigma}{4}} \\
&\leq \frac{-\eta\sigma}{4} + \prn*{1-L\eta}^2\prn*{\E x_0 + \frac{\eta\sigma}{4}}
\end{align}
This completes both cases of \eqref{eq:one-step-lemma-x3}.
\end{proof}

\begin{lemma}\label{lem:sgd-many-steps}
Fix $L,\eta,\sigma > 0$ such that $L\eta \leq \frac{1}{2}$ and let $k \geq 2$. Let $x_0$ denote a random initial point with $\E x_0 \leq 0$ and let $x_k$ denote the $k$th iterate of stochastic gradient descent on $g_L$ with fixed stepsize $\eta$ intialized at $x_0$. Then
\begin{gather*}
\E x_k \leq \begin{cases}
\frac{-\eta\sigma}{48} & \E x_0 \leq \frac{-\eta\sigma}{48} \\
\frac{-\eta\sigma}{4} + \prn*{1-L\eta}^{k/2}\prn*{\E x_0 + \frac{\eta\sigma}{4}} & \E x_0 \in \left(\frac{-\eta\sigma}{48}, 0\right]
\end{cases}
\end{gather*}
\end{lemma}
\begin{proof}
The idea of this proof is simple: $k$ steps of SGD initialized at some point $x_0$ is equivalent to doing two steps of SGD initialized at $x_0$ to get $x_2$, then doing two more steps initialized at $x_2$ to get $x_4$, and so forth until $k$ steps have been completed. The only minor complication is if $k$ is odd, in which case we start by doing three steps initialized at $x_{0}$ to get $x_3$ and continue in steps of two.

We will consider two cases, either $\E x_0 \leq \frac{-\eta\sigma}{48}$ or $\E x_0 \in \left(\frac{-\eta\sigma}{48}, 0\right]$. 
In the first case, $\E x_0 \leq \frac{-\eta\sigma}{48}$, if $k$ is even then by Lemma \ref{lem:sgd-two-step}
\begin{equation}
\E x_0 \leq \frac{-\eta\sigma}{48} \implies \E x_2 \leq \frac{-\eta\sigma}{48} \implies \E x_4 \leq \frac{-\eta\sigma}{48} \implies \dots \implies \E x_k \leq \frac{-\eta\sigma}{48}
\end{equation}
If $k$ is odd then
\begin{equation}
\E x_0 \leq \frac{-\eta\sigma}{48} \implies \E x_3 \leq \frac{-\eta\sigma}{48} \implies \E x_5 \leq \frac{-\eta\sigma}{48} \implies \dots \implies \E x_k \leq \frac{-\eta\sigma}{48}
\end{equation}

In the second case, $\E x_0 \in \left(\frac{-\eta\sigma}{48}, 0\right]$. Then, when $k$ is even, by repeatedly invoking Lemma \ref{lem:sgd-two-step} we get
\begin{align}
\E x_2 &\leq \frac{-\eta\sigma}{4} + \prn*{1-L\eta}\prn*{\E x_0 + \frac{\eta\sigma}{4}} \\
\E x_4 &\leq \frac{-\eta\sigma}{4} + \prn*{1-L\eta}\prn*{\E x_2 + \frac{\eta\sigma}{4}} \leq \frac{-\eta\sigma}{4} + \prn*{1-L\eta}^2\prn*{\E x_0 + \frac{\eta\sigma}{4}} \\
\E x_6 &\leq \frac{-\eta\sigma}{4} + \prn*{1-L\eta}\prn*{\E x_4 + \frac{\eta\sigma}{4}} \leq \frac{-\eta\sigma}{4} + \prn*{1-L\eta}^3\prn*{\E x_0 + \frac{\eta\sigma}{4}} \\
&\ \ \vdots \\
\E x_k &\leq \frac{-\eta\sigma}{4} + \prn*{1-L\eta}^{k/2}\prn*{\E x_0 + \frac{\eta\sigma}{4}}
\end{align}
The same argument applies when $k$ is odd (using the bound on $\E x_3$) to prove 
\begin{equation}
\E x_k \leq \frac{-\eta\sigma}{4} + \prn*{1-L\eta}^{(k+1)/2}\prn*{\E x_0 + \frac{\eta\sigma}{4}} \leq \frac{-\eta\sigma}{4} + \prn*{1-L\eta}^{k/2}\prn*{\E x_0 + \frac{\eta\sigma}{4}}\qedhere
\end{equation}
\end{proof}

\begin{lemma}\label{lem:local-sgd-on-stochastic-coordinate}
Let $K \geq 2$ and let $\hat{x}$ be the output of local-SGD$(K,R,M)$ on $F$ using a fixed stepsize $\eta \leq \frac{1}{2L}$ and initialized at zero. Then
\[
\E\brk*{\frac{L}{2}\prn*{\prn*{\hat{x}_3 - c}^2 + \pp{\hat{x}_3 - c}^2}} 
\geq \frac{L\eta^2\sigma^2}{4608} \indicatorb{\eta \leq \frac{1}{2L}}\indicatorb{c \geq \frac{\eta\sigma}{48} \lor \eta \geq \frac{2}{LRK}}
\]
\end{lemma}
\begin{proof}
Since each coordinate evolves independently when optimizing $F$ using local-SGD, we can ignore the first two coordinates and focus only on the third. Observe that using local-SGD$(K,R,M)$ on $F$ with a fixed stepsize $\eta$ and initialized at zero to obtain $\hat{x}_3$ is exactly equivalent to using local-SGD$(K,R,M)$ on $g_{L}$ with the same fixed stepsize $\eta$ and initialized at $-c$. The different initialization is due to the fact that the local-SGD dynamics do not change with the change of variables $x - c \rightarrow x$. Let $\bar{x}_r$ denote the averaged iterate of local-SGD$(K,R,M)$ initialized at $-c$ with stepsize $\eta$ after the $r$th round of communication and let $x_{r,k,m}$ denote its $k$th iterate during the $r$th round of communication on the $m$th machine. We will start by proving that when $\eta \leq \frac{1}{2L}$ and \emph{either} $c \geq \frac{\eta\sigma}{8}$ \emph{or} $\eta \geq \frac{2}{LRK}$ then
\begin{equation}
    \E \hat{x}_3 - c = \E\bar{x}_R \leq \frac{-\eta\sigma}{48}
\end{equation}

Consider first the case $\E x_0 = -c \leq \frac{-\eta\sigma}{48}$. Then by Lemma \ref{lem:sgd-many-steps}
\begin{equation}
\E x_0 = -c \leq \frac{-\eta\sigma}{48} \implies \E x_{1,K,m} \leq \frac{-\eta\sigma}{48} \ \ \forall m
\end{equation}
therefore
\begin{equation}
\E \bar{x}_1 = \E\brk*{\frac{1}{M}\sum_{m=1}^M x_{1,K,m}} \leq \frac{-\eta\sigma}{48}
\end{equation}
Repeatedly applying Lemma \ref{lem:sgd-many-steps} shows that for each $r$
\begin{equation}
\E \bar{x}_r \leq \frac{-\eta\sigma}{48} \implies \E x_{r+1,K,m} \leq \frac{-\eta\sigma}{48} \implies \E \bar{x}_{r+1} = \E\brk*{\frac{1}{M}\sum_{m=1}^M x_{r+1,K,m}} \leq \frac{-\eta\sigma}{48}
\end{equation}
We conclude $\E \bar{x}_R \leq \frac{-\eta\sigma}{48}$.

Consider instead the case that $\E x_0 = -c \in \left( \frac{-\eta\sigma}{48}, 0 \right]$  and $\eta \geq \frac{2}{LRK}$. Then, by Lemma \ref{lem:sgd-many-steps}
\begin{equation}
\E x_0 = -c \in \left( \frac{-\eta\sigma}{48}, 0 \right] \implies \E x_{1,K,m} \leq \frac{-\eta\sigma}{4} + \prn*{1-L\eta}^{K/2}\prn*{\frac{\eta\sigma}{4} - c} \ \ \forall m
\end{equation}
and so
\begin{equation}
\E \bar{x}_1 = \E\brk*{\frac{1}{M}\sum_{m=1}^M x_{1,K,m}} \leq \frac{-\eta\sigma}{4} + \prn*{1-L\eta}^{K/2}\prn*{\frac{\eta\sigma}{4} - c}
\end{equation}
Again, we can repeatedly apply Lemma \ref{lem:sgd-many-steps} to show 
\begin{align}
\E \bar{x}_2 &\leq \frac{-\eta\sigma}{4} + \prn*{1-L\eta}^{K/2}\prn*{\E \bar{x}_1 + \frac{\eta\sigma}{4}} 
\leq \frac{-\eta\sigma}{4} + \prn*{1-L\eta}^{2K/2}\prn*{\frac{\eta\sigma}{4} - c} \\
\E \bar{x}_3 &\leq \frac{-\eta\sigma}{4} + \prn*{1-L\eta}^{K/2}\prn*{\E \bar{x}_2 + \frac{\eta\sigma}{4}} 
\leq \frac{-\eta\sigma}{4} + \prn*{1-L\eta}^{3K/2}\prn*{\frac{\eta\sigma}{4} - c} \\
&\ \ \vdots\\
\E \bar{x}_R 
&\leq \frac{-\eta\sigma}{4} + \prn*{1-L\eta}^{RK/2}\prn*{\frac{\eta\sigma}{4} - c} \\
&\leq -\prn*{1 - \prn*{1-L\eta}^{RK/2}}\frac{\eta\sigma}{4} \\
&\leq -\prn*{1 - \prn*{1-\frac{2}{RK}}^{RK/2}}\frac{\eta\sigma}{4} \\
&\leq \frac{\eta\sigma}{48}
\end{align}
These inequalities hold only as long as $\E \bar{x}_r > \frac{-\eta\sigma}{48}$. But, if for some $r$, $\E \bar{x}_r \leq \frac{-\eta\sigma}{48}$ then $\E \bar{x}_R \leq \frac{-\eta\sigma}{48}$ by the same argument as above.
We conclude that
\begin{equation}
\E \bar{x}_R \leq
\frac{-\eta\sigma}{48}\indicatorb{\eta \leq \frac{1}{2L}}\indicatorb{c \geq \frac{\eta\sigma}{48} \lor \eta \geq \frac{2}{LRK}}
\end{equation}
Since $\E \hat{x}_3 - c = \E \bar{x}_R$, by Jensen's inequality
\begin{align}
\E\brk*{\frac{L}{2}\prn*{\prn*{\hat{x}_3 - c}^2 + \pp{\hat{x}_3 - c}^2}} 
&\geq \frac{L}{2}\prn*{\prn*{\E \bar{x}_R}^2 + \pp{\E \bar{x}_R}^2} \\
&\geq \frac{L\eta^2\sigma^2}{4608} \indicatorb{\eta \leq \frac{1}{2L}}\indicatorb{c \geq \frac{\eta\sigma}{48} \lor \eta \geq \frac{2}{LRK}}
\end{align}
\end{proof}

We now analyze the progress of SGD on the first two coordinates of $F$ in the following lemma:
\begin{lemma}\label{lem:sgd-on-deterministic-coordinates}
Let $\hat{x}$ be the output of local-SGD$(K,R,M)$ on $F$ using a fixed stepsize $\eta$ and initialized at zero. Then with probability 1,
\[
\frac{\mu}{2}\prn*{\hat{x}_1 - b}^2 \geq \frac{\mu b^2}{8}\indicatorb{\eta < \frac{1}{2\mu KR}}
\]
and
\[
\frac{H}{2}\prn*{\hat{x}_2 - b}^2 \geq \frac{H b^2}{2}\indicatorb{\eta > \frac{2}{H}}.
\]
\end{lemma}
\begin{proof}
Since the stochastic gradient estimator has no noise along the first and second coordinates, and since the separate coordinates evolve independently, $\hat{x}_1$ is exactly the output of $KR$ steps of deterministic gradient descent with fixed stepsize $\eta$ on the univariate function $x \mapsto \frac{\mu}{2}\prn*{x - b}^2$. Similarly, $\hat{x}_2$ is the output of $KR$ steps of deterministic gradient descent with fixed stepsize $\eta$ on $x \mapsto \frac{H}{2}\prn*{x - b}^2$. Thus,
\begin{multline}
x_1\ind{t+1} - b = x_1\ind{t} - b - \eta\mu\prn*{x_1\ind{t} - b} 
\implies \hat{x}_1 = b + \prn*{1-\eta\mu}^{KR}\prn*{x_1\ind{0} - b} = b\prn*{1 - \prn*{1-\eta\mu}^{KR}}
\end{multline}
Thus, if $\eta < \frac{1}{2\mu KR}$, then 
\begin{equation}
\hat{x}_1 \leq b\eta\mu KR< \frac{b}{2} \implies \frac{\mu}{2}\prn*{\hat{x}_1 - b}^2 \geq \frac{\mu b^2}{8}\indicatorb{\eta < \frac{1}{2\mu KR}}
\end{equation}
Similarly,
\begin{multline}
x_2\ind{t+1} - b = x_2\ind{t} - b - \eta H\prn*{x_2\ind{t} - b} 
\implies \hat{x}_2 - b = \prn*{1-\eta H}^{KR}\prn*{x_2\ind{0} - b} = - b\prn*{1-\eta H}^{KR}
\end{multline}
Thus, if $\eta > \frac{2}{H}$, then 
\begin{equation}
\abs{\hat{x}_2 - b} \geq b \implies \frac{H}{2}\prn*{\hat{x}_2 - b}^2 \geq \frac{H b^2}{2}\indicatorb{\eta > \frac{2}{H}}
\end{equation}
\end{proof}

Combining Lemmas \ref{lem:local-sgd-on-stochastic-coordinate} and \ref{lem:sgd-on-deterministic-coordinates}, we are ready to prove the theorem:
\lowerbound*
\begin{proof}
Consider optimizing the objective $F$ defined in \eqref{eq:lower-bound-construction-app} using the stochastic gradient oracle \eqref{eq:lower-bound-stochastic-gradient-app} initialized at zero and using a fixed stepsize $\eta$. The variance of the stochastic gradient oracle is equal to $\sigma^2$. This function is $\max\crl*{\mu, H, 2L}$-smooth, and $\min\crl*{\mu, H, L}$-strongly convex. We will be choosing $L = \frac{H}{4}$ and $\mu \in \left[\lambda, \frac{H}{16}\right]$ so that $F$ is $H$-smooth and $\lambda$-strongly convex.
Finally, the objective $F$ is minimized at the point $x^* = [b,b,c]^\top$ and $F(x^*) = 0$. This point has norm $\nrm*{x^*} = \sqrt{2b^2 + c^2}$ we will choose $b = c = \frac{B}{\sqrt{3}}$ so that $\nrm*{x^*} = B$.

By Lemma \ref{lem:local-sgd-on-stochastic-coordinate}, the output of local-SGD$(K,R,M)$, $\hat{x}$ satisfies
\begin{equation}
\E\brk*{\frac{L}{2}\prn*{\prn*{\hat{x}_3 - c}^2 + \pp{\hat{x}_3 - c}^2}} 
\geq \frac{L\eta^2\sigma^2}{4608} \indicatorb{\eta \leq \frac{1}{2L}}\indicatorb{c \geq \frac{\eta\sigma}{48} \lor \eta \geq \frac{2}{LRK}}
\end{equation}
By Lemma \ref{lem:sgd-on-deterministic-coordinates}, the output of local-SGD$(K,R,M)$, $\hat{x}$ satisfies
\begin{equation}
\frac{\mu}{2}\prn*{\hat{x}_1 - b}^2 + \frac{H}{2}\prn*{\hat{x}_2 - b}^2 
\geq \frac{\mu b^2}{8}\indicatorb{\eta < \frac{1}{2\mu KR}} + \frac{H b^2}{2}\indicatorb{\eta > \frac{2}{H}}
\end{equation}
Combining these, we have
\begin{multline}
\E F(\hat{x}) - \min_x F(x) 
\geq \frac{\mu b^2}{8}\indicatorb{\eta < \frac{1}{2\mu KR}} + \frac{H b^2}{2}\indicatorb{\eta > \frac{2}{H}} + \frac{L\eta^2\sigma^2}{4608} \indicatorb{\eta \leq \frac{1}{2L}}\indicatorb{\eta \leq \frac{48c}{\sigma} \lor \eta \geq \frac{2}{LRK}}
\end{multline}
Consider two cases: first, suppose that $\eta \not\in \brk*{\frac{1}{2\mu KR}, \frac{2}{H}}$. Then,
\begin{equation}
\E F(\hat{x}) - \min_x F(x)
\geq \min\crl*{\frac{\mu b^2}{8}, \frac{H b^2}{2}} 
= \frac{\mu b^2}{8}\label{eq:lower-bound-eq1}
\end{equation}
Suppose instead that $\eta \in \brk*{\frac{1}{2\mu KR}, \frac{2}{H}}$. Since $L = \frac{H}{4}$, $\eta \leq \frac{2}{H} \leq \frac{1}{2L}$. Similarly, since $\mu \leq \frac{H}{16} = \frac{L}{4}$, $\eta \geq \frac{1}{2\mu KR} \geq \frac{2}{LRK}$. Therefore, $\eta \in \brk*{\frac{1}{2\mu KR}, \frac{2}{H}}$ implies
\begin{align}
\E F(\hat{x}) - \min_x F(x) 
&\geq \min_{\eta \in \brk*{\frac{1}{2\mu KR}, \frac{2}{H}}} \frac{L\eta^2\sigma^2}{4608} \indicatorb{\eta \leq \frac{1}{2L}}\indicatorb{\eta \leq \frac{48c}{\sigma} \lor \eta \geq \frac{2}{LRK}} \\
&= \min_{\eta \in \brk*{\frac{1}{2\mu KR}, \frac{2}{H}}} \frac{L\eta^2\sigma^2}{4608} \\
&= \frac{L\sigma^2}{18432\mu^2K^2R^2}\label{eq:lower-bound-eq2}
\end{align}
Combining \eqref{eq:lower-bound-eq1} and \eqref{eq:lower-bound-eq2} yields
\begin{equation}
\E F(\hat{x}) - \min_x F(x)
\geq \min\crl*{\frac{\mu B^2}{24}, \frac{H\sigma^2}{73728\mu^2K^2R^2}}
\end{equation}
This statement holds for any $\mu \in \brk*{\lambda, \frac{H}{16}}$. Consider three cases: first, suppose $\mu = \prn*{\frac{H\sigma^2}{3072B^2K^2R^2}}^{1/3} \in \brk*{\lambda, \frac{H}{16}}$. Then
\begin{equation}
\E F(\hat{x}) - \min_x F(x)
\geq \frac{H^{1/3}\sigma^{2/3}B^{4/3}}{350K^{2/3}R^{2/3}}
\end{equation}
Consider next the case that $\prn*{\frac{H\sigma^2}{3072B^2K^2R^2}}^{1/3} > \frac{H}{16}$ $\implies$ $\frac{\sigma^2}{192B^2K^2R^2} > \frac{H^2}{256}$ and choose $\mu = \frac{H}{16}$. Then
\begin{equation}
\E F(\hat{x}) - \min_x F(x)
\geq \min\crl*{\frac{H B^2}{384}, \frac{H\sigma^2}{73728K^2R^2\cdot\frac{H^2}{256}}} 
= \frac{H B^2}{384}
\end{equation}
Finally, consider the case that $\prn*{\frac{H\sigma^2}{3072B^2K^2R^2}}^{1/3} < \lambda$ and choose $\mu = \lambda$. Then,
\begin{equation}
\E F(\hat{x}) - \min_x F(x)
\geq \min\crl*{\frac{\lambda B^2}{24}, \frac{H\sigma^2}{73728\lambda^2K^2R^2}} = \frac{H\sigma^2}{73728\lambda^2K^2R^2}
\end{equation}
Combining these cases completes the proof.
\end{proof}

\end{document}